\documentclass{article}

\usepackage{microtype}
\usepackage{graphicx}

\usepackage{subcaption}
\usepackage{booktabs} 

\usepackage{hyperref}
\usepackage{etoolbox}
\makeatletter
\patchcmd\@combinedblfloats{\box\@outputbox}{\unvbox\@outputbox}{}{\errmessage{\noexpand patch failed}}
\makeatother

\usepackage[arxiv]{icml2019}

\usepackage[utf8]{inputenc} 
\usepackage[T1]{fontenc}    
\usepackage{url}            
\usepackage{booktabs}       
\usepackage{amsfonts,amsmath,amsthm}       
\usepackage{nicefrac}       
\usepackage{microtype}      
\usepackage[algo2e,ruled,noline]{algorithm2e}
\usepackage{pgfplots}
\usepackage{blindtext}
\usepackage{bm}
\usepackage{color}
\pgfplotsset{compat=newest}

\newcommand{\unit}[1]{\bm{[}#1\bm{]}}
\newcommand{\A}{\boldsymbol{A}}

\newcommand{\U}{\boldsymbol{U}}
\newcommand{\W}{\boldsymbol{W}}

\renewcommand{\u}{\boldsymbol{u}}

\newcommand{\w}{\boldsymbol{w}}
\newcommand{\x}{\boldsymbol{x}}
\newcommand{\y}{\boldsymbol{y}}

\newcommand{\hy}{\widehat{y}}
\newcommand{\hby}{\widehat{\y}}
\newcommand{\regret}{R}
\DeclareMathOperator*{\argmin}{argmin}
\newcommand{\vt}{\widetilde{v}}

\newtheorem{theorem}{Theorem}[section]
\newtheorem*{theorem*}{Theorem}
\newtheorem{lemma}[theorem]{Lemma}

\icmltitlerunning{Adaptive scale-invariant online algorithms for learning linear models}
\begin{document}

\twocolumn[
\icmltitle{Adaptive scale-invariant online algorithms for learning linear models}



\begin{icmlauthorlist}
\icmlauthor{Micha{\l} Kempka}{put}
\icmlauthor{Wojciech Kot{\l}owski}{put}
\icmlauthor{Manfred K. Warmuth}{google}
\end{icmlauthorlist}

\icmlaffiliation{put}{
Poznan University of Technology, Poznan, Poland}
\icmlaffiliation{google}{Google Inc. Z\"urich \& UC Santa Cruz}

\icmlcorrespondingauthor{Micha{\l} Kempka}{mkempka@cs.put.poznan.pl}
\icmlcorrespondingauthor{Wojciech Kot{\l}owski}{wkotlowski@cs.put.poznan.pl}
\icmlcorrespondingauthor{Manfred K. Warmuth}{manfred@ucsc.edu}
\icmlkeywords{Machine Learning, ICML}

\vskip 0.3in
]



\printAffiliationsAndNotice{} 

\begin{abstract}
We consider online learning with linear models, where the algorithm predicts on sequentially revealed instances (feature vectors), and is compared against the best linear function (comparator) in hindsight. Popular algorithms in this framework, such as Online Gradient Descent (OGD), have parameters (learning rates), which ideally should be tuned based on the scales of the features and the optimal comparator, but these quantities only become available at the end of the learning process. In this paper, we resolve the tuning problem by proposing online algorithms making predictions which are invariant under arbitrary rescaling of the features. The algorithms have no parameters to tune, do not require any prior knowledge on the scale of the instances or the comparator, and achieve regret bounds matching (up to a logarithmic factor) that of OGD with optimally tuned separate learning rates per dimension, while retaining comparable runtime performance.
\end{abstract}

\section{Introduction}

We consider the problem of online learning with linear models, in which at
each trial $t=1,\ldots,T$, the algorithm receives an input instance (feature vector)
$\x_t \in \mathbb{R}^d$, upon which it predicts $\hy_t \in \mathbb{R}$. Then,
the true label $y_t$ is revealed and the algorithm suffers loss $\ell(y_t,\hy_t)$,
convex in $\hy_t$.
The goal of the algorithm is have its cumulative loss not much larger to that of
any linear predictor of the form 
$\x \mapsto \x^\top \u$ for $\u \in \mathbb{R}^d$, i.e.
to have small \emph{regret} against any comparator $\u$.
This problem encompasses linear regression and classification (with convex surrogate
losses) and has been extensively studied
in numerous past works \citep{llw-ollf-91,clw-wqlbop-96,ShaiBook,Hazan_OCO}.

One of the most popular algorithms in this framework is Online Gradient Descent (OGD)
\citep{clw-wqlbop-96}. Its predictions are given by $\hy_t = \x_t^\top \w_t$
for a weight vector $\w_t \in \mathbb{R}^d$
updated using a simple rule 
\begin{equation}
\w_{t+1} = \w_t - \eta \nabla_t,
\label{eq:OGD_update}
\end{equation}
where $\nabla_t$ is
a (sub)gradient of the loss $\ell(y_t,\hy_t)$ 
at $\w_t$, and $\eta$ is a parameter of the algorithm, called
the learning rate. With the optimal ``oracle'' tuning of $\eta$ (which involves the norm
of the comparator $\u$ and of the observed gradients, unknown to the algorithm in advance), 
OGD would achieve 
a bound on the regret
against $\u$ of order $\|\u\| \sqrt{\sum_t \|\nabla_t\|^2}$ \citep{gd}.
Unfortunately, this bound might be very poor if the features have distinct scales.
To see that first note that $\nabla_t$ is proportional to $\x_t$ due to linear dependence 
of $\hy_t$ on $\x_t$. Now, let $\u$ be the comparator which minimizes the total
loss (assume such exists). If we scale the first coordinate of each $\x_t$ by a factor
of $c$, 
the first coordinate of the optimal comparator $\u$ will scale down by a factor of
$c^{-1}$, so its prediction and (optimal) loss remain the same, and the bound above will
in general become worse by a factor of $\max\{c,c^{-1}\}$ \citep{Ross_etal2013UAI}.
This is a well known issue with gradient descent, and is usually solved by prior
normalization of the features. However, such pre-processing step cannot be done 
in an online setting.

The problem described above becomes apparent if we make an analogy from physics and
imagine that all features have physical \emph{units}. 
In particular, if we assigned a unit $\unit{x_i}$ to feature $i$, and assumed for simplicity 
that the prediction and the label are unitless (as in, e.g., classification), 
the corresponding coordinate of the weight vector 
would need to have unit $1/\unit{x_i}$. 
However, the units in the OGD update \eqref{eq:OGD_update} are
mismatched, because $\nabla_{t,i}$ has unit $\unit{x_i}$ (as $\nabla_t$ is proportional
to $\x_{t}$), while $w_{t,i}$ has unit $1/\unit{x_i}$;
even assigning a unit to $\eta$ does
not help as a single number cannot compensate different units. A reasonable solution to
this ``unit clash'' problem is to use \emph{one learning rate per dimension}, i.e.
to modify \eqref{eq:OGD_update} to:
\begin{equation}
  w_{t+1,i} = w_{t,i} - \eta_i \nabla_{t,i}, \qquad i=1,\ldots,d.
\label{eq:OGD_per_coordinate_update}
\end{equation}
If we choose the oracle tuning of the learning rates to minimize the regret
against comparator $\u$, it follows that
$\eta_i^* = |u_i| / \sqrt{\sum_t \nabla_{t,i}^2}$ which results in the regret bound
of order $\sum_i |u_i| \sqrt{\sum_t \nabla_{t,i}^2}$, better than the
bound obtained with a single learning rate. 
Interestingly, the unit of $\eta_i^*$ becomes $1/[x_i]^2$, which fixes
the ``unit clash'' in \eqref{eq:OGD_per_coordinate_update} and
makes the scaling issues go away (as now scaling the $i$-the feature
by any factor $c$ will be compensated by scaling down $u_i$ by $c^{-1}$). 
Unfortunately,
it is infeasible in practice to separately tune a single learning rate per dimension 
(oracle tuning requires the knowledge the comparator and all future gradients).

\paragraph{Our contribution.} In this paper we provide adaptive online algorithms 
which for any comparator $\u$ achieve regret bounds matching, up to logarithmic
factors, that of OGD with 
optimally tuned separate learning rates per dimension. Note that as we want to
capture arbitrary feature scales and comparators, our bounds come \emph{without
any prior assumptions on the magnitude of instances $\x_t$, comparator $\u$, or even 
predictions $\x_t^\top \u$}, as has been commonly assumed in the past work
(we do, however, assume the Lipschitzness of the loss with respect to the prediction,
which is satisfied for various popular loss functions, such as logistic, hinge or
absolute losses\footnote{Lipschitzness 
\emph{does not} imply any bound on the 
gradients $\nabla_t$ which are proportional to feature vectors: $\nabla_t = g_t \x_t$
for some $g_t \in \mathbb{R}$; 
it only implies a bound on the proportionality constant $g_t$.}). Our algorithms achieve their bounds without the need to tune any
hyperparameters, and have runtime performance of $O(d)$
per iteration, which is the same as that of OGD. As a by-product of being
adaptive to the scales of the instances and the comparator, the proposed algorithms
are \emph{scale-invariant}: their predictions are invariant
under arbitrary rescaling of individual features \citep{Ross_etal2013UAI}. More precisely,
after multiplying the $i$-th coordinate of all input instances by a 
fix scaling factor $a_i$, $x_{t,i} \mapsto a_i x_{t,i}$ for all $t$, the predictions
of the algorithms remain the same: they are
independent on the units in which the instance vectors are expressed (in particular,
they do do not require any prior normalization of the data). To achieve our goals,
the design of our algorithms heavily rely on techniques recently developed in 
adaptive online learning \citep{StreeterMcMahan2012,OrabonaPal2016NIPS,CutkoskyBoahen2017COLT,CutkoskyOrabona2018COLT}.

The first algorithm achieves a regret bound which depends on instances only relative
to the scale of the comparator, 
through products of the form $|u|_i \sqrt{\max_t |x_{t,i}|^2
+ \sum_t \nabla_{t,i}^2}$ for $i=1,\ldots,d$, 
similarly as in the bound of OGD with per-dimension learning
rates (with additional maximum over feature values, which is usually much smaller
than the sum over squared gradients). As the algorithm can be sometimes a bit conservative in its predictions, we also
introduce a second algorithm which is more aggressive in decreasing its cumulative loss;
the price to pay is a regret bound which mildly (logarithmically) depends on 
ratios between the largest and the first non-zero input value for each coordinate. While
these quantities can be made arbitrarily large in the worst case, it is unlikely to
happen in practice. We test both algorithms in a computational study on several real-life
data sets and show that without any need to tune parameters, they
are competitive to popular online learning methods, which are allowed to tune 
their learning rates to optimize the test set performance.

\paragraph{Related work.}

Our work is rooted from a long line of research
on regret minimizing online algorithms \cite{clw-wqlbop-96,eg,PLGbook}.
Most of the proposed methods have ``range factors'' present both in 
the algorithm and in the bound: it is typically assumed
that some prior knowledge on the range of the comparator and the gradients is given,
which allows the algorithm to tune its parameters appropriately. For instance,
assuming $\|\u\| \leq U$ and $\|\nabla_t\| \leq G$ for all $t$, OGD \eqref{eq:OGD_update}
with learning rate $\eta = U / (G \sqrt{T})$ achieves $O(UG \sqrt{T})$ regret bound.

More recent work on adaptive algorithms aims to get rid of
these range factors. In particular, with a prior bound on the comparator norm,
it is possible to adapt to the unknown range of the gradients \citep{adagrad,OrabonaPal2015ALT},
whereas having a prior bound on all future gradients, one can adapt 
to the unknown norm of the comparator \citep{StreeterMcMahan2012,McMahanAbernethy2013,Orabona2014,OrabonaPal2016NIPS,cocob,CutkoskyOrabona2018COLT}.
In particular, using reduction methods proposed by \citet{CutkoskyOrabona2018COLT}, one can get
a bound matching OGD with separate learning rate per dimension, but this requires to know 
$\max_{t,i} |\nabla_{t,i}|$ in advance.
Interestingly, \citet{CutkoskyBoahen2017COLT} have shown that in online convex optimization
it is not possible to adapt to both unknown gradient range
and unknown comparator norm at the same time. Here, we circumvent this negative result by
exploiting the fact that the input instance $\x_t$ is available
\emph{ahead} of prediction and therefore can be used to construct $\hy_t$
(this idea was first discovered in the context of linear regression 
\cite{Vovk, aw}).

Scale-invariant algorithm has been been studied by
\citet{Ross_etal2013UAI,Orabona_etal2015} in a setup very similar to ours. Their algorithms,
however, require a prior knowledge on the largest per-coordinate comparator's prediction,
$\max_{t,i} |u_i x_{t,i}|$, whereas their bounds scale with relative ratios between the largest
and the first non-zero input value for each coordinate (the bound of our second algorithm also depends
on these quantities but only in a logarithmic way).
\citet{Luo_etal2016,affine_invariant} considered 
even a more general setup of invariance under linear transformation of features (of which our
invariance is a special case if the transformation is diagonal),
but a prior knowledge of $\max_t |\x_t^\top \u|$ must be available,
and the resulting algorithms are second-order methods.
The closest to our work are the results by \citet{Kotlowski_ALT2017}, which concern
the same setup, general invariance under linear transformations, and, similarly to us, make no prior range assumptions. Their
bounds, however, do not scale with gradients $\nabla_{t,i}^2$ (as in the optimal OGD bound), 
but with the \emph{size of the features} $x_{t,i}^2$ (multiplied by the Lipschitz constant
of the loss), which upper-bounds $\nabla_{t,i}^2$ and can become much larger.
For instance, in the ``noise-free'' case, when some comparator $\u$ has zero loss,
the algorithm playing sufficiently close to $\u$ can inflict arbitrarily small gradients,
while the sum of squared feature values will still grow linearly in $t$. 

The goal of scale-invariance seems to go hand in hand
with a requirement for the updates to avoid unit clashes
and this connection was the motivating idea for our work.
In the most basic case, assume you want to design online algorithms for linear regression 
$$\w_{t+1}=\w_t-\eta(\x_t^\top \w_t -y_t)\x_t$$
that are to be robust to scaling
the input vectors $\x_t$ by a single positive constant factor.
In this case $\unit{\eta}$ should be $1/\unit{\|\x_t\|^2}$.
Interestingly enough, good tunings of the learning rates $\eta$ often ``fix the
units'': the properly tuned learning rates for the linear regression updates 
employed in \cite{clw-wqlbop-96,eg} have units $1/\unit{\|\x_t\|^2}$.
In this paper we focus on robustness to independently scaling the individual components 
$x_{t,i}$ of the input vectors by positive factors. This requires privatized learning rates
$\eta_i$ with the property that $\unit{\eta_i}=1/\unit{x_{t,i}}^2$.
Our paper focuses on this case because of efficiency concerns. 
However there is a third case (more expensive)
where we want robustness to independent scaling and
rotation of the input vectors $\x_t$.
Now $\bm{\eta}$ must be a matrix parameter (playing a similar role to a Hessian) and 
if the instances are pre-multiplied by a fixed invertible 
$\A$, then the tuned learning rate matrix of
the new instances must become $\bm\eta \A^{-1}$, thus correcting for the
pre-multiplication with $\A$. The updates of
\citet{Luo_etal2016,affine_invariant,Kotlowski_ALT2017},
as well as the Newton algorithm, have this form, but they are all
second order algorithms with runtime of at least $O(d^2)$ per trial.

\section{Problem Setting}

Our online learning protocol is defined as follows.
In each trial $t=1,\ldots,T$, the algorithm receives 
an input instance $\x_t \in \mathbb{R}^d$, on which it predicts 
$\hy_t \in \mathbb{R}$; we will always assume linear predictions
$\hy_t = \x_t^\top \w_t$, where
$\w_t \in \mathbb{R}^d$ is allowed to depend on $\x_t$.
Then, the output label $y_t \in \mathcal{Y}$ is revealed, 
and the algorithm suffers loss $\ell(y_t,\hy_t)$. 
As we make no assumptions about the label set $\mathcal{Y}$, 
in what follows we incorporate $y_t$
into the loss function and use
$\ell_t(\hy)$ to denote $\ell(y_t, \hy)$.
The performance of the algorithm is measured by means of the \emph{regret}:
\[
  \regret_T(\u) = \sum_{t=1}^T \ell_t(\x_t^\top \w_t) - \sum_{t=1}^T \ell_t(\x_t^\top \u),
\]
which is the difference between the cumulative loss 
of the algorithm and that of a fixed, arbitrarily chosen, 
comparator weight vector $\u \in \mathbb{R}^d$
(for instance, $\u$ can be the minimizer of the cumulative loss on the 
whole data sequence, if such exists).
\begin{table}
\begin{center}
  \begin{tabular}{cccc}
\toprule
Loss function & $\ell(y,\hy)$ & $\partial_{\hy} \ell(y,\hy)$ & $L$ \\
\hline
logistic & $\ln\left(1 + e^{-y \hy}\right)$ & $\frac{-y}{1+e^{y\hy}}$ & 1 \\
hinge & $\max\{0,1-y\hy\}$ & $-y \boldsymbol{1}[y\hy \leq 1]$ & 1 \\
absolute & $|\hy - y|$ & $\mathrm{sgn}(\hy-y)$ & 1 \\
\bottomrule
\end{tabular}
\end{center}
\caption{$L$-Lipschitz loss functions for classification and regression. 
$\boldsymbol{1}[\cdot]$ denotes an indicator function.}
\label{tab:losses}
\end{table}

We assume that for any $t$, $\ell_t(\hy)$ is convex and $L$-Lipschitz; the latter implies that
the (sub)derivative of the loss is bounded, $|\partial\ell_t(\hy)| \leq L$.
Table \ref{tab:losses} lists three popular losses with these properties.
Throughout the paper, we assume $L=1$ without loss of generality.
Our setup can be considered as a variant of online convex optimization
\citep{ShaiBook,Hazan_OCO}, with the main difference in $\x_t$ being observed before
prediction.

We use a standard argument
exploiting the convexity of the loss to bound
$\ell_t(\hy') \geq \ell_t(\hy) + \partial \ell_t(\hy) (\hy' - \hy)$ for any $\hy, \hy' \in \mathbb{R}$.
Substituting $\hy = \w_t^\top \x_t$ and $\hy' = \u^\top \x_t$, and
denoting $g_t = \partial \ell_t(\hy_t) \in [-1,1]$ for each $t$, the regret is upper-bounded by:
\begin{equation}
\regret_T(\u) \leq \sum_{t=1}^T g_t \x_t^\top (\w_t - \u),
\label{eq:gradient_trick}
\end{equation}
where $|g_t| \leq 1$ follows from the Lipschitzness of the loss.  Note that
$g_t \x_t$ is equal to $\nabla_t = \nabla_{\w_t} \ell_t(\x_t^\top \w_t)$, the
(sub)gradient of the loss with respect to the weight vector $\w_t$.  Thus, we
can bound the regret with respect to the original convex loss $\ell_t(\x_t^\top
\w)$ by upper-bounding its linearized version $g_t \x_t^\top \w$ on the
right-hand side of \eqref{eq:gradient_trick}.

Consider running Online Gradient Descent (OGD) algorithm on this problem, as
defined in \eqref{eq:OGD_per_coordinate_update}, i.e. we let the algorithm
have a separate learning rate per dimension.
When initialized at $\w_1 = \boldsymbol{0}$, OGD achieves the regret bound:
\[ 
\regret_T(\u) ~\leq~
\sum_{i=1}^d \left(\frac{u_i^2}{2\eta_i} + \frac{\eta_i}{2} S^2_{T,i} \right),
\] 
where we introduced $S^2_{t,i} = \sum_{j \leq t} \nabla_{j,i}^2 = \sum_{j \leq t}
(g_j x_{j,i})^2$. This is a slight generalization of a standard textbook
bound \citep[see, e.g.,][]{Hazan_OCO}, proven in Appendix
\ref{appendix:OGD} for completeness.
Tuning the learning rates to minimize the bound results in $\eta_i =
\frac{|u_i|}{S_{t,i}}$, and the bound simply becomes:
\begin{equation}
\regret_T(\u) \leq \sum_{i=1}^d |u_i| S_{T,i}.
\label{eq:ideal_bound}
\end{equation}
Such tuning is, however, not directly feasible as it would require knowing
the comparator and the future gradients in hindsight. 
The goal of this work is to design adaptive online algorithms 
which for any comparator $\u$, and any data sequence $\{(\x_t,y_t)\}_{t=1}^T$, without
any prior knowledge on their magnitudes, achieve \eqref{eq:ideal_bound} up to logarithmic
factors.

An interesting property of bound \eqref{eq:ideal_bound} is that it captures
a natural symmetry of our linear framework. Given a data sequence,
let $\u$ be the minimizer of the cumulative loss,
$\u = \argmin_{\w} \sum_t \ell_t(\x_t^\top \w)$ (assume
such exists). If
we apply a coordinate-wise transformation $x_{t,i} \mapsto a_i x_{t,i}$  
simultaneously to all input instances ($t=1,\ldots,T$)
for any positive scaling factors $a_1,\ldots,a_d$, the minimizer of the loss will
undergo the inverse transformation $u_{i} \mapsto a_i^{-1} u_{i}$ to keep its
predictions $\x_t^\top \u$, and thus its cumulative loss, \emph{invariant}. Indeed,
$\u$ minimizes $\sum_t \ell_t(\x_t^\top \w)$ if and only if $\A^{-1} \u$
minimizes $\sum_t \ell_t((\A \x_t)^\top \w)$ for $\A = \mathrm{diag}(a_1,\ldots,a_d)$.
Thus, when \eqref{eq:ideal_bound} is evaluated at the loss minimizer, it becomes
invariant under any such scale transformation.

The invariance of predictions of the optimal comparator leads to the definition
of scale-invariant algorithms. We call a learning algorithm \emph{scale-invariant} 
if its behavior (sequence of predictions) is invariant
under arbitrary rescaling of individual features \citep{Ross_etal2013UAI,Kotlowski_ALT2017}. 
More precisely, if we apply a transformation $x_{t,i} \mapsto a_i x_{t,i}$ 
simultaneously to all instances,
the predictions of the algorithm $\hy_1,\ldots,\hy_T$ remain the same
as on the original data sequence. 
Scale-invariant algorithm are thus independent on the ``units'' 
in which the instance vectors are expressed on each feature,
and do not require any prior normalization of the data. 
Interestingly, OGD defined in \eqref{eq:OGD_per_coordinate_update} is not a scale invariant
algorithm, but becomes one under the optimal tuning of its learning rates. The algorithms
presented in the next section will turn out to be scale-invariant, essentially as a
by-product of adaptiveness to arbitrary scale of the comparator and the instances, required
to achieve \eqref{eq:ideal_bound}.

\emph{Remark:} As noted in the introduction, scale invariance can be generalized to arbitrary
linear invertible transformations $\x_t \mapsto \A \x_t$. Unfortunately, 
this leads to second-order
algorithms \citep{Luo_etal2016,affine_invariant,Kotlowski_ALT2017}, with the complexity
at least $\Theta(d^2)$ per trial.

\section{Scale-invariant algorithms}
\label{sec:algorithms}

\paragraph{Motivation.} We first briefly describe the motivating idea behind
the construction of the algorithms. We start with rewriting the right hand
side of \eqref{eq:gradient_trick} to get: 
\[
\regret_T(\u) \leq \sum_{i=1}^d \Big( \underbrace{\sum_{t=1}^T g_t x_{t,i} (w_{t,i} - u_i)}_{\stackrel{def}{=} \tilde{R}_{T,i}(u_i)} \Big),
\]
so that
it decouples coordinate-wise and it suffices to separately bound each term $\tilde{R}_{T,i}(u_i)$ 
$i=1,\ldots,d$. As we aim to get close to \eqref{eq:ideal_bound},
we want for each $i$ a bound of the form
$\tilde{R}_{T,i}(u_i) \leq B(u_i S_{T,i}) + c_T$, for some function $B(\cdot)$
plus a potential additional overhead $c_T$
(to exactly get \eqref{eq:ideal_bound} we could set $B(x) = |x|$ and $c_T=0$,
but this turns out to be unachievable without any prior knowledge on the comparator).
Using $G_{t,i} = -\sum_{j \leq t} g_j x_{j,i}$
to denote the cumulative negative gradient coordinate, 
such bound can be equivalently written as:
\[
\sum_{t=1}^T g_t x_{t,i} w_{t,i} + G_{T,i} u_i -B(u_i S_{T,i}) \leq c_T.
\]
Now, the key idea is to note that the bound must hold for any comparator $u_i$, therefore
it must hold if we take a supremum over $u_i$ on the left-hand side,
$\sup_{u_i} \{ G_{T,i} u_i - B(u_iS_{T,i})\}$. 
To evaluate this supremum, we note that under variable change $x = u_i S_{T,i}$ it
becomes equivalent to $\sup_{x} \{ G_{T,i}/S_{T,i} x - B(x)\}$.
Recalling the definition of 
the \emph{Fenchel conjugate} of a function $f(x)$, defined
as $f^*(\theta) = \sup_x \{\theta x - f(x)\}$ \citep{BoydVandenberghe04},
we see that the supremum can be evaluated to $B^*(G_{T,i}/S_{T,i})$.
Thus, the unknown comparator has been eliminated from the picture, and the 
algorithm can be designed to satisfy:
\[
\sum_{t=1}^T g_t x_{t,i} w_{t,i} + B^*(G_{T,i}/S_{T,i}) \leq c_T,
\]
for every data sequence. In fact, we construct our algorithms by proceeding
in the reverse direction: starting with an appropriate function $\psi$ 
playing the role of $B^*$ (which we call a \emph{potential}) and getting bound expressed
by means of its conjugate $\psi^*$. What we just described is known as
\emph{regret-reward duality} and has been successfully used in adaptive online learning
\citep{StreeterMcMahan2012,McMahanOrabona2014,OrabonaPal2016NIPS}.

As already briefly mentioned, achieving \eqref{eq:ideal_bound}, which corresponds to a bound
with $B(x)=|x|$, is actually not possible: a negative result by 
\citet{StreeterMcMahan2012} implies that the best one can hope for is 
$B(x) = O(|x| \sqrt{\ln(|x|)})$. We will show that our algorithm achieve a bound of a slightly
weaker form $B(x) = O(|x| \ln(|x|))$, but still giving only a logarithmic overhead comparing
to \eqref{eq:ideal_bound}.

\paragraph{Algorithms.} 
We propose two scale-invariant algorithms presented as Algorithm \ref{alg:one}
($\mathrm{ScInOL}_1$ from \underline{Sc}ale-\underline{In}variant \underline{O}nline
\underline{L}earning)
and Algorithm \ref{alg:two} ($\mathrm{ScInOL}_2$).
They require $O(d)$ operations per trial and thus match OGD in
the computational complexity. Both algorithms keep track of the
negative cumulative gradients $G_{t,i} = -\sum_{j=1}^t g_j x_{j,i}$, sum
of squared gradients $S^2_{t,i} = \sum_{j=1}^t (g_j x_{j,i})^2$, and the
maximum encountered input values $M_{t,i} = \max_{j \leq t} |x_{j,i}|$.
The weight formula is written to highlight that
the cumulative gradients are only accessed
through a unitless quantity $\frac{G_{t-1,i}}{\sqrt{S_{t-1,i}^2 + M_{t,i}^2}}$;
an additional factor $\frac{1}{\sqrt{S_{t-1,i}^2 + M_{t,i}^2}}$ in the weights is to compensate
for $x_{t,i}$ in the prediction. To simplify the pseudocode we use
the convention that $\frac{0}{0} = 0$ and $\frac{c}{0} = \infty$ for $c > 0$. 
Note that since in each trial $t$, the algorithms have access to the input feature
vector $\x_t$ \emph{before} the prediction, they are able to update
$M_{t,i}$ prior to computing the weight vector $\w_t$. 
Both algorithms decompose into $d$ one-dimensional copies, one per each feature, which
are coupled only by the values of $g_t$. 
Both algorithm have a parameter $\epsilon$, but it only affect the constants and 
is set to $1$ in the experiments. Scale invariance of the algorithms is verified
in Appendix \ref{appendix:scale_invariance}. 

\begin{algorithm2e}[t]%
\DontPrintSemicolon
\SetAlgoNoEnd
\SetKwInOut{Initialization}{Initialization}
\textbf{Initialize:} $S^2_{0,i}, G_{0,i}, M_{0,i} \! \leftarrow 0, \beta_{0,i} 
\leftarrow \epsilon$; $i=1,\ldots,d$ \\
 \For{$t=1,\ldots,T$}{
   Receive $\x_t \in \mathbb{R}^d$\;
   \For{$i=1,\ldots,d$}{
     $M_{t,i} \leftarrow \max \{M_{t-1,i}, |x_{t,i}|\}$\;
     $\beta_{t,i} \leftarrow \min\{\beta_{t-1,i}, 
     \epsilon (S^2_{t-1,i} + M_{t,i}^2)/(x^2_{t,i} t)\}$ \;
     $w_{t,i} = \frac{\beta_{t,i} \mathrm{sgn}(\theta_{t,i})}{2\sqrt{S_{t-1,i}^2 + M_{t,i}^2}}
        \Big(e^{|\theta_{t,i}|/2} - 1 \Big)$ \;
        \qquad where $\theta_{t,i} = \frac{G_{t-1,i}}{\sqrt{S^2_{t-1,i} + M^2_{t,i}}}$}
   Predict with $\hy_t = \x_t^\top \w_{t,i}$, receive loss $\ell_t(\hy_t)$ and compute $g_t = \partial_{\hy_t} \ell_t(\hy_t)$\;
   \For{$i=1,\ldots,d$}{
     $G_{t,i} \leftarrow G_{t-1,i} - g_t x_{t,i}$ \;
     $S^2_{t,i} \leftarrow S^2_{t-1,i} + (g_t x_{t,i})^2$\;
   }
 }
\caption{$\text{ScInOL}_1 (\epsilon=1)$}%
\label{alg:one}
\end{algorithm2e}%

\paragraph{ScInOL$_1$.}
The algorithm is based on a potential
$\psi_{t,i}(x) = \beta_{t,i} (e^{|x|/ (2 \hat{S}_{t,i})} - \frac{|x|}{2\hat{S}_{t,i}} -1
)$ with
 $\hat{S}_{t,i} = \sqrt{S_{t,i}^2 + M_{t,i}^2}$.
The weight $w_{t,i}$ is chosen in such a way
that the loss of the algorithm at trial $t$ is upper-bounded by the change in the potential,
for any choice of $x_{t,i} \in \mathbb{R}$ and $g_t \in [-1,1]$:
\begin{equation}
  w_{t,i} g_t x_{t,i} \leq \psi_{t-1,i}(G_{t-1,i}) - \psi_{t,i}(G_{t,i}) + \delta_{t,i},
  \label{eq:scinol_1_change_in_potential_bound}
\end{equation}
where $\delta_{t,i}$ is a small additional overhead. The algorithm resembles FreeRex by
\cite{CutkoskyBoahen2017COLT}, because it actually uses the same functional form of the potential. The choice of
the weight looks \emph{almost} like a derivative of a potential function $\psi_{t-1,i}(x)$ at $x = G_{t-1,i}$, 
but it differs slightly in using $M_{t,i}$ rather than $M_{t-1,i}$ in its definition.
This prior update of $M_{t,i}$ let the algorithm account for potentially very large value of $x_{t,i}$
and avoid incurring too much loss. The coefficients $\beta_{t,i}$ multiplying the potential are chosen
to be a nonincreasing sequence, which at the same time keeps the overhead $\delta_t$ upper-bounded by $\frac{\epsilon}{t}$,
in order to to avoid terms in the regret bound depending on ratios between feature values
and get $\sum_t \delta_{t,i} \leq \epsilon(1 + \ln T)$. Summing \eqref{eq:scinol_1_change_in_potential_bound}
over trials and using using $\psi_{0,i}(G_0) = 0$ gives:
\[
 \sum_t w_{t,i} g_t x_{t,i} - \sum_t \delta_{t,i}  \leq - \psi_{T+1,i}(G_{T,i}).
\]
Using the convexity of $\psi_{T+1,i}$ we can rewrite it by means of its Fenchel conjugate,
$\psi_{T+1,i}(G_{T,i}) = \sup_{u}\{G_{T,i} u - \psi^*_{T+1,i}(u)\}$, which in turn
can be bounded as:
\[
  \psi^*_{T,i}(u) 
  \leq 2|u| \hat{S}_{T,i} \ln \left(1 + 2|u|\beta^{-1}_{T,i} \hat{S}_{T,i}\right)
\]
Summing over features $i=1,\ldots,d$,
bounding $\beta_{T,i} \geq (\epsilon T)^{-1}$,
and using \eqref{eq:gradient_trick} gives:
\begin{theorem}
\label{thm:alg_one}
For any $\u \in \mathbb{R}^d$ the regret of $\mathrm{ScInOL}_1$ is upper-bounded by:
\begin{align*}
  \regret_T(\u)
  &\leq \sum_{i=1}^d \!\bigg(\!2|u_i|\hat{S}_{T,i} \ln \! 
  \Big(1 \!+\! \frac{2|u_i| \hat{S}_{T,i} T}{\epsilon}\Big) + \epsilon(1 + \ln T )\!\bigg) \\
  &= \sum_{i=1}^d \tilde{O}(|u_i| \hat{S}_{T,i}),
\end{align*}
where $\hat{S}_{T,i} = \sqrt{S_{T,i}^2 + M_{T,i}^2}$ and $\tilde{O}(\cdot)$ hides the
constants and logarithmic factors.
\end{theorem}
The full proof of Theorem \ref{thm:alg_one} is given in Appendix \ref{appx:alg_one}.
Note that the bound depends on the scales of the features only relative to the
comparator weights $\u$ through quantities $|u_i| \hat{S}_{T,i}$, and is equivalent
to the optimal OGD bound \eqref{eq:ideal_bound} up to logarithmic factors.

\begin{algorithm2e}[t]%
\DontPrintSemicolon
\SetAlgoNoEnd
\SetKwInOut{Initialization}{Initialization}
\textbf{Initialize:} $S^2_{0,i}, G_{0,i}, M_{0,i} \leftarrow 0, 
\eta_{0,i} \leftarrow \epsilon; i=1,\ldots,d$ \\
 \For{$t=1,\ldots,T$}{
   Receive $\x_t \in \mathbb{R}^d$\;
   \For{$i=1,\ldots,d$}{
     $M_{t,i} \leftarrow \max \{M_{t-1,i}, |x_{t,i}|\}$\;
    $w_{t,i} = 
    \frac{\mathrm{sgn}(\theta_{t,i}) \min\{|\theta_{t,i}|,1\}}{2 \sqrt{S_{t-1,i}^2 + M_{t,i}^2}} \eta_{t-1,i}$ \\
     \qquad where $\theta_{t,i} = \frac{G_{t-1,i}}{\sqrt{S^2_{t-1,i} + M^2_{t,i}}}$}
   Predict with $\hy_t = \x_t^\top \w_{t,i}$, receive loss $\ell_t(\hy_t)$ and compute $g_t = \partial_{\hy_t} \ell_t(\hy_t)$\;
   \For{$i=1,\ldots,d$}{
     $G_{t,i} \leftarrow G_{t-1,i} - g_t x_{t,i}$ \;
     $S^2_{t,i} \leftarrow S^2_{t-1,i} + (g_t x_{t,i})^2$\;
     $\eta_{t,i} \leftarrow \eta_{t-1,i} - g_t x_{t,i} w_{t,i}$
   }
 }
\caption{$\text{ScInOL}_2 (\epsilon=1)$}%
\label{alg:two}
\end{algorithm2e}%

\paragraph{ScInOL$_2$.} The algorithm described in the previous section is designed to 
achieve a regret bound 
which depends on instances only relative to the scale of the comparator,
no matter how extreme are the ratios $|x_{t,i}| / M_{t-1,i}$ between the new inputs
and previously observed maximum feature values. We have observed that this can
make the behavior of the algorithm too conservative, due to guarding
against the worst-case instances.
Therefore we introduce a second algorithm, 
which is more aggressive in decreasing its cumulative loss;
the price to pay is a regret bound which mildly depends on 
ratios between feature values. The algorithm has
a multiplicative flavor and resembles a family of Coin Betting algorithms recently
developed by \citet{OrabonaPal2016NIPS,cocob,CutkoskyOrabona2018COLT}.  

The algorithm is based on a potential function
$\psi_{t,i}(x) = e^{\frac{1}{2} h(x / \hat{S}_{t,i})}$, where:
\[
  h(y) = \left\{
    \begin{array}{ll}
      \frac{1}{2} y^2 & \qquad \text{for~} |y| \leq 1, \\
      |y| - \frac{1}{2} & \qquad \text{for~} |y| > 1.
    \end{array}
    \right.
\]

Function $h(y)$ interpolates between between the quadratic (for $|y| \leq 1$) and
absolute value (otherwise). It is easy to check that $h(y) \geq |y| - \frac{1}{2}$
for all $y$.

By the definition, $\eta_{t,i} = \epsilon - \sum_{j \leq t} g_t x_{t,i} w_{t,i}$
is (up to $\epsilon$) the cumulative negative loss of the algorithm (``reward''). 
The weights are chosen in order to guarantee the relative increase in the 
reward lower-bounded by the
relative increase in the potential:
\[
  \frac{\eta_{t,i}}{\eta_{t-1,i}} \geq \frac{\psi_{t,i}(G_{t,i})}{\psi_{t-1,i}(G_{t-1,i})}
  e^{-\delta_{t,i}},
\]
where $\delta_{t,i}$ is an overhead which can be controlled. Taking the product
over trials $t=1,\ldots,T$ and using $\psi_{0,i}(G_{0,i})=1$ gives
$\eta_T \geq \epsilon \psi_{T,i}(G_{T,i}) e^{-\Delta_T}$, where $\Delta_{T,i} = \sum_t \delta_{t,i}$.
Using the definition of $\eta_{T,i}$, this translates to:
\begin{align*}
  \sum_{t} g_t x_{t,i} w_{t,i}  \leq &
  \epsilon - \epsilon \psi_{T,i}(G_{T,i}) e^{-\Delta_{T,i}} \\
  \leq & \epsilon - \epsilon e^{- \Delta_{T,i} - \frac{1}{4}} e^{|G_{T,i}|/(2 \hat{S}_{T,i})},
\end{align*}
where we used $h(y) \geq |y| - \frac{1}{2}$.
Denote the function on the r.h.s. by
$f(x) = \epsilon e^{- \Delta_{T,i} - \frac{1}{4}} e^{|x|/(2 \hat{S}_{T,i})}$. Using convexity
of $f(x)$, we can express it by means of its Fenchel conjugate,
$f(G_{T,i}) = \sup_u \{G_{T,i} u - f^*(u)\}$,
for which we have the following bound \citep{Orabona2013}:
\[
  f^*(u) \leq 2|u| \hat{S}_{T,i}\left( \ln \left(2 \epsilon^{-1} |u| \hat{S}_{T,i} e^{\frac{1}{4} + \Delta_{T,i}}  \right) - 1 \right).
\]
Unfortunately, it turns out that $\Delta_{T,i}$ 
can be $\Omega(T)$ in the worst case, which makes the bound linear in $T$. 
We can, however, bound $\Delta_{T,i}$ in a data-dependent way by:
\[
  \Delta_{T,i} \leq \ln \left(\frac{\hat{S}^2_{T,i}}{x_{\tau_i,i}^2} \right)
\]
where $\tau_i$ is the first trial in which $|x_{t,i}| \neq 0$.
As $\hat{S}^2_{T,i} \leq (T+1) \max_t x_{t,i}^2$, the bound
involves the ratio between the largest and the first non-zero input value. While
being vacuous in the worst case, this quantity is likely not to be excessively
large for non-adversarial data encountered in practice, and it is moreover hidden under
the logarithm in the bound (a similar quantity is analyzed by \citet{Ross_etal2013UAI},
where its magnitude is bounded with high probability for data received in a random
order).
Following along the steps from the previous section, we end up with the following bound:
\begin{theorem}
\label{thm:alg_two}
For any $\u \in \mathbb{R}^d$ the regret of $\mathrm{ScInOL}_2$ is upper-bounded by:
\[
  \regret_T(\u)
  \leq d \epsilon +
  \sum_{i=1}^d 2|u_i| \hat{S}_{T,i} \left( \ln (3|u_i| \hat{S}^3_{T,i} \epsilon^{-1} /
  x_{\tau_i,i}^2) - 1 \right),
\]
    where $\hat{S}_{T,i} = \sqrt{S_{T,i}^2 + M_{T,i}^2}$
 and
$\tau_i = \min\{t \colon |x_{t,i}| \neq 0\}$.
\end{theorem}
The proof of Theorem \ref{thm:alg_two} is given in Appendix \ref{appx:alg_two}.

\section{Experiments}

\subsection{Toy example}
\begin{figure*}[!ht]
    \centering
    \begin{subfigure}[b]{0.45\textwidth}
        \centering
        \includegraphics[width=\textwidth]{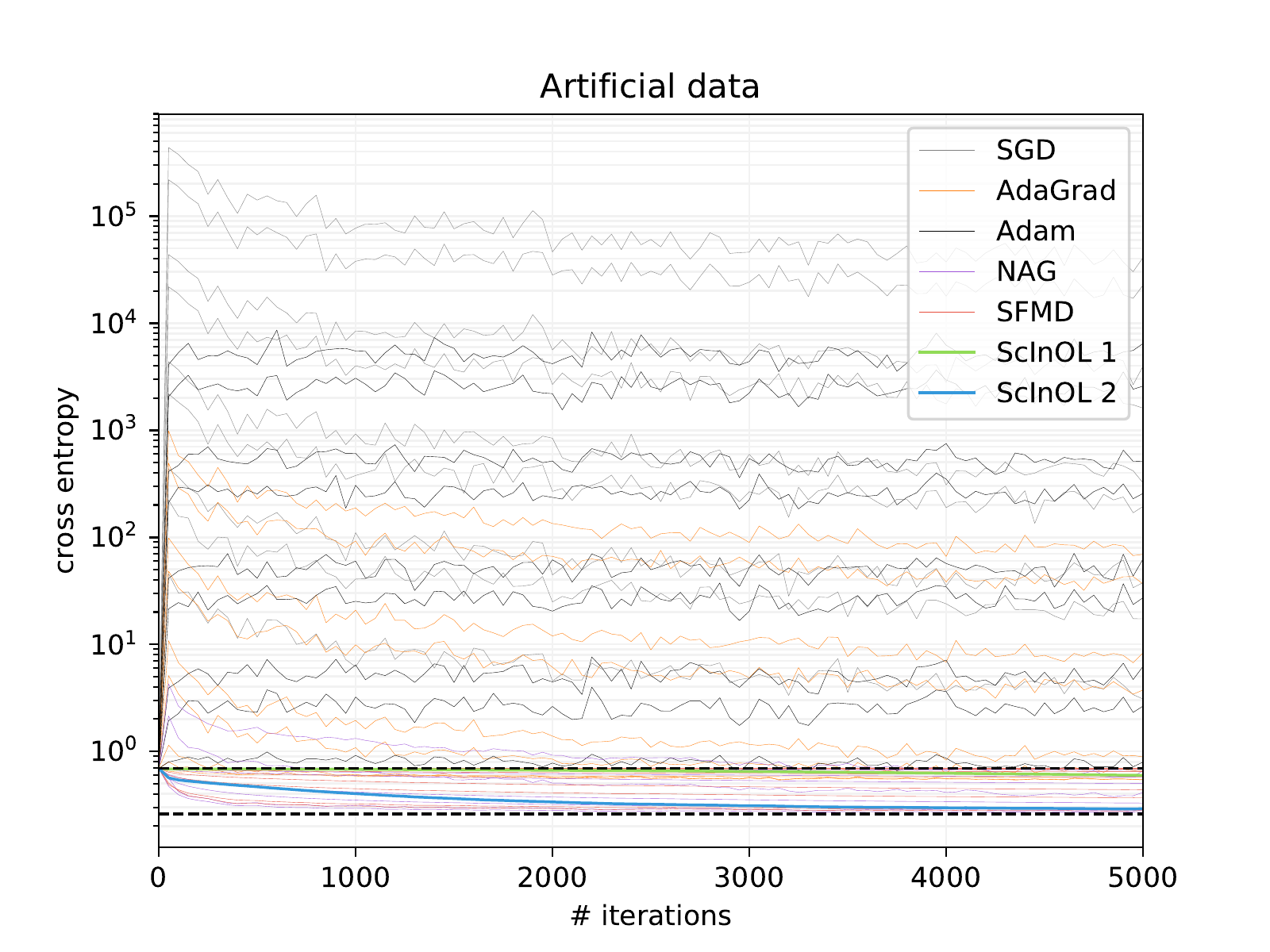}
        \caption{Entire view of the results with logarithmic scale. }
        \label{fig:art_all}
    \end{subfigure}
    \begin{subfigure}[b]{0.45\textwidth}
        \centering
        \includegraphics[width=\textwidth]{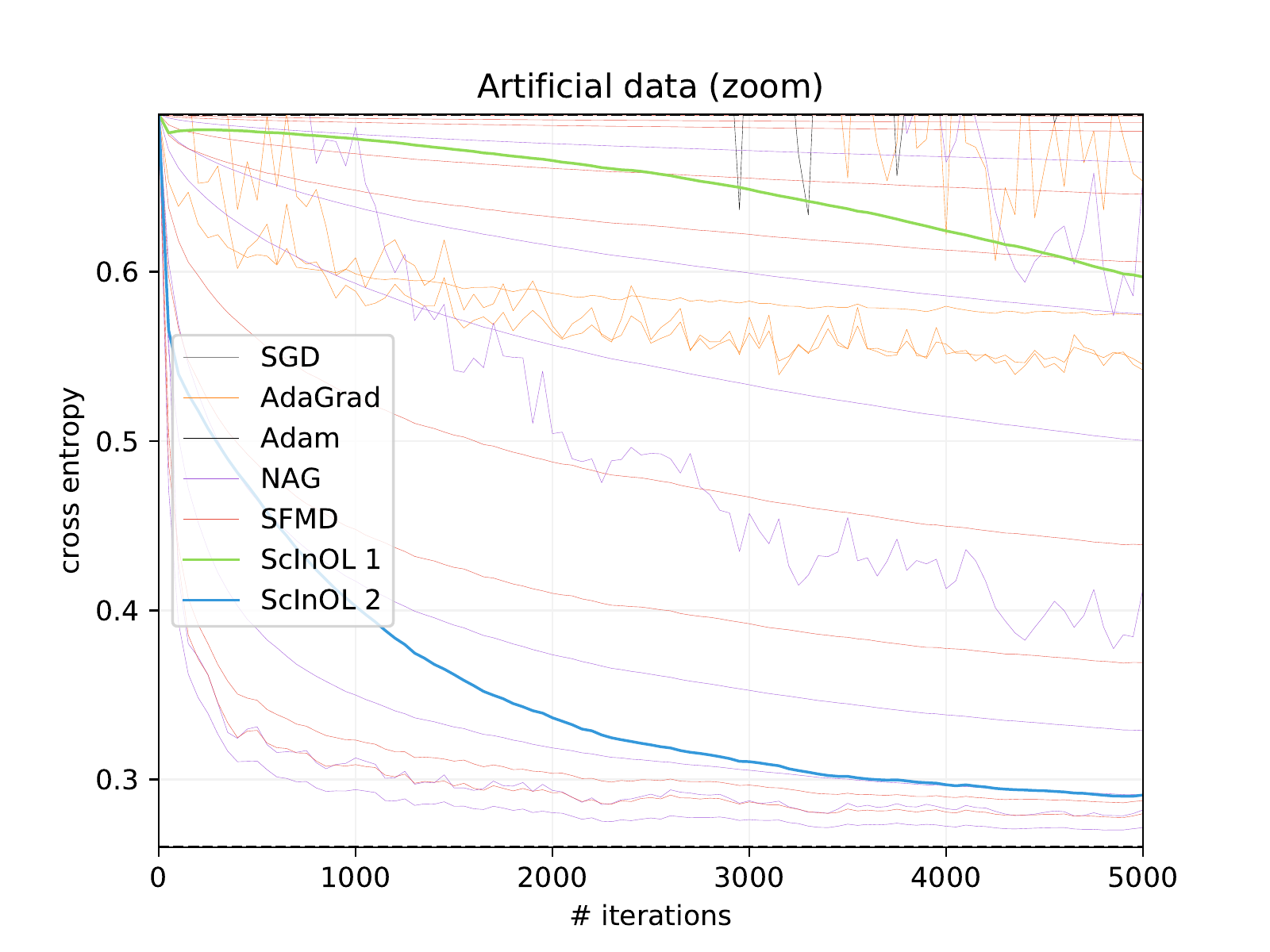}
        \caption{Zoomed-in fragment showing best performers.}
        \label{fig:art_zoom}
    \end{subfigure}
    \caption{Average cross entropy from test set for the toy example.}
    \label{fig:art}
\end{figure*}

To empirically demonstrate a need for scale invariance we tested our algorithms against some popular adaptive variants of OGD. 
The tests were run on a simple artificial binary classification
dataset that mildly exaggerates relative magnitudes of features (however still keeps them in reasonable ranges).
The dataset contains 21 real features, values of which are drawn from normal distributions $N(0,\sigma_i)$, where $\sigma_i = 2^{i-11}$ ($i=1,\ldots,21$), so that
the scales of features vary from $2^{-10}$ to $2^{10}$ (the ratio of the largest and the smallest scale is of order $10^6$). 
Binary class labels where drawn from a Bernoulli distribution with $\Pr(y=1|\x) = \mathrm{sigmoid}(\x^\top \u)$ where $u_i=\pm \frac{1}{\sigma_i}$ 
with signs chosen uniformly at random. Note that $\u$ is set to compensate the scale of features and keep the predictions function $\x^\top \u$ to be of order of unity.
We have drawn 5\;000 training examples and 100\;000 test examples. We repeated the experiment on 10 random training sets to decrease the variation
of the results.

The algorithms were trained by minimizing the logistic loss (cross entropy loss) in an online fashion. Following similar experiments
in the past papers concerning online methods \citep{DBLP:journals/corr/KingmaB14,Ross_etal2013UAI,cocob}, we report the average loss on
the test set (after every 50 iterations) rather than the regret. We tested the following algorithms: OGD with learning rate
decaying as $\eta_t = \eta / \sqrt{t}$ (called SGD here from \emph{stochastic gradient descent}), AdaGrad \citep{adagrad}, Adam \citep{DBLP:journals/corr/KingmaB14}, two scale-invariant
algorithms from past work: NAG (Normalized Adaptive Gradient) \citep{Ross_etal2013UAI} and Scale-free Mirror Descent by \citet{Orabona_etal2015} (SFMD),
and algorithms from this work.  All algorithms except ours have a learning rate parameter, which in each case was set to values from 
$\{0.001,0.005,0.01,0.05,0.1,0.5,1,5,10\}$ (results concerning all learning rates were reported). We implemented our algorithms in Tensorflow
and used existing implementations whenever it was possible.

Figure \ref{fig:art} shows average cross entropy measured on test set as a function of the number of iterations. 
Lines of the same color show results for the same algorithm but with different learning rates. Figure \ref{fig:art_all} uses logarithmic scale for $y$ axis: note
the extreme values of the loss for most of the non-invariant algorithms. In fact, the two black dashed lines mark the loss achieved by the best possible model $\u$ (lower line)
and a model with zero weight vector (upper line), so that every method above the upper dashed line does something worse than such a trivial baseline.
Figure \ref{fig:art_zoom} shows only the fragment between dashed lines using linear scale for $y$ axis. 

The results clearly show that algorithms which are not invariant to feature scales (SGD, Adam, and AdaGrad) 
are unable to achieve any reasonable result for any choice of the learning rate (most of the time performing much worse than the zero vector).
This is because a single learning rate is unable to compensate all
feature scales at the same time.
The scale invariant algorithms, NAG and SFMD, perform much better (achieving the best overall results), but their behavior still depends on the learning rate tuning.
Among our algorithms, ScInOL$_1$ slowly decreases its loss moving away from the initial zero solution, but it is clearly too slow in this problem.
On the other hand, ScInOL$_2$ was able to achieve descent results without any tuning at all.

\subsection{Linear Classification}
\begin{figure*}[!ht]
    \centering
    \begin{tabular}{ccc}
        \includegraphics[width=.33\textwidth]{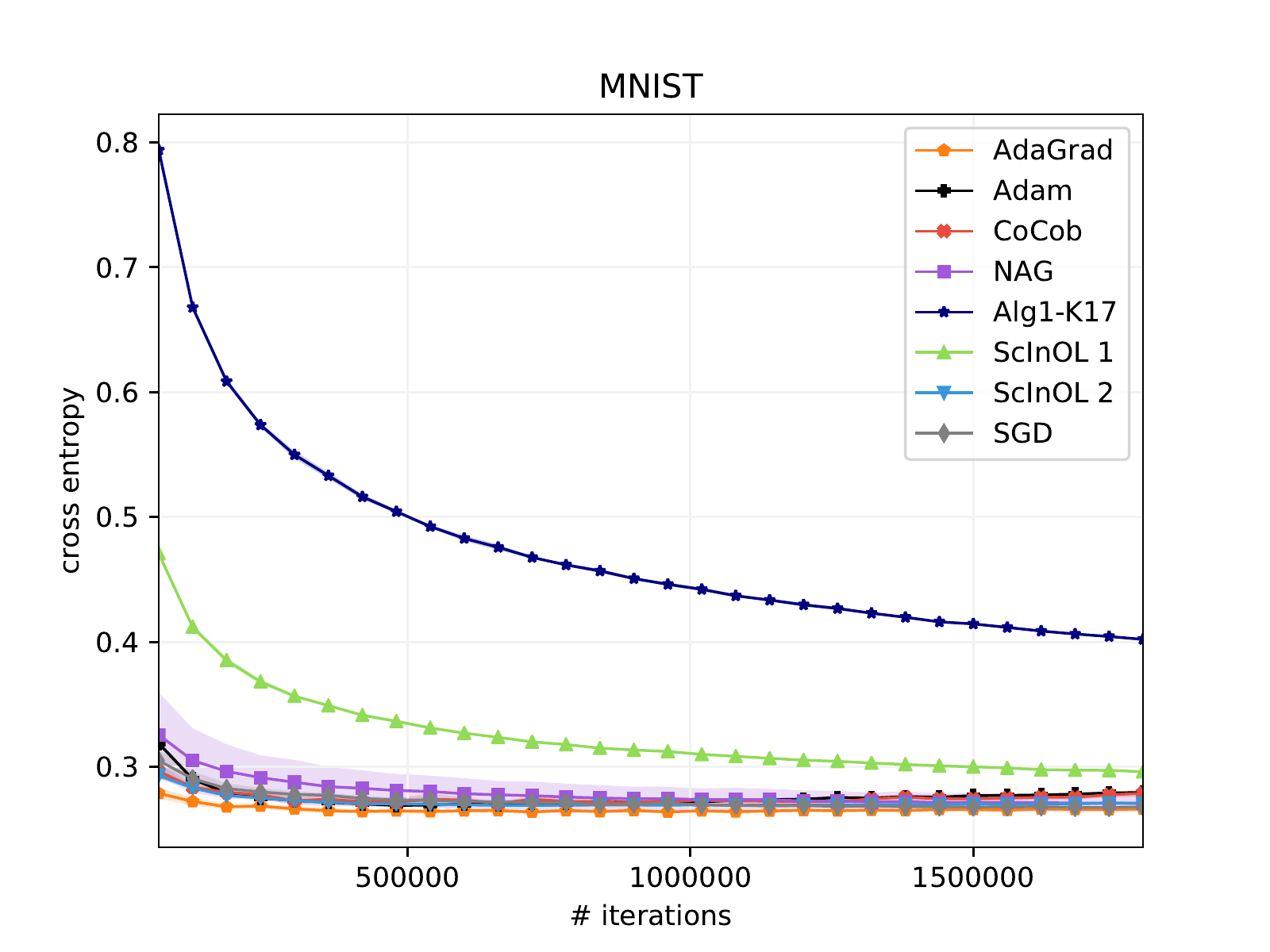} &
        \includegraphics[width=.33\textwidth]{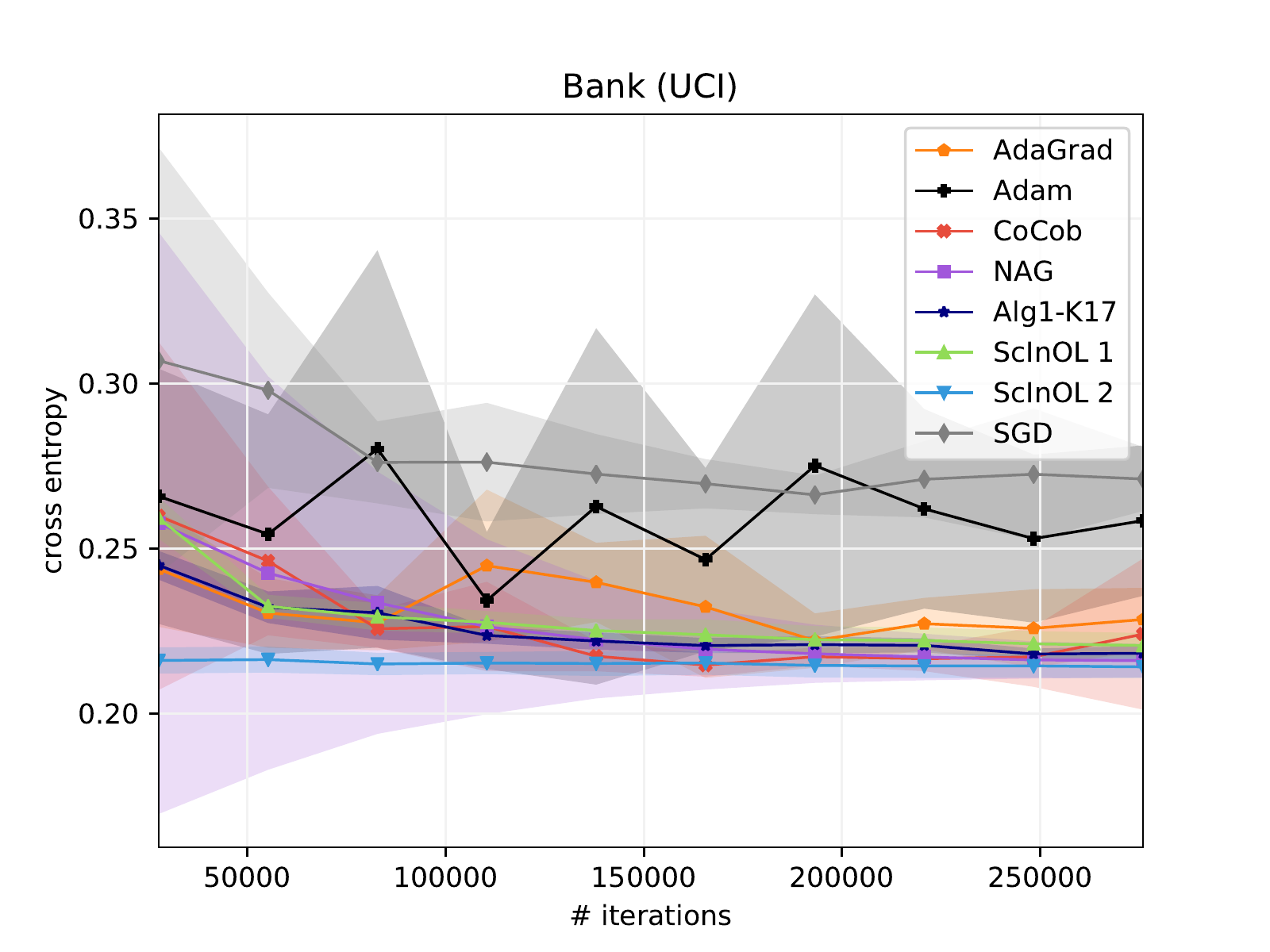} &
        \includegraphics[width=.33\textwidth]{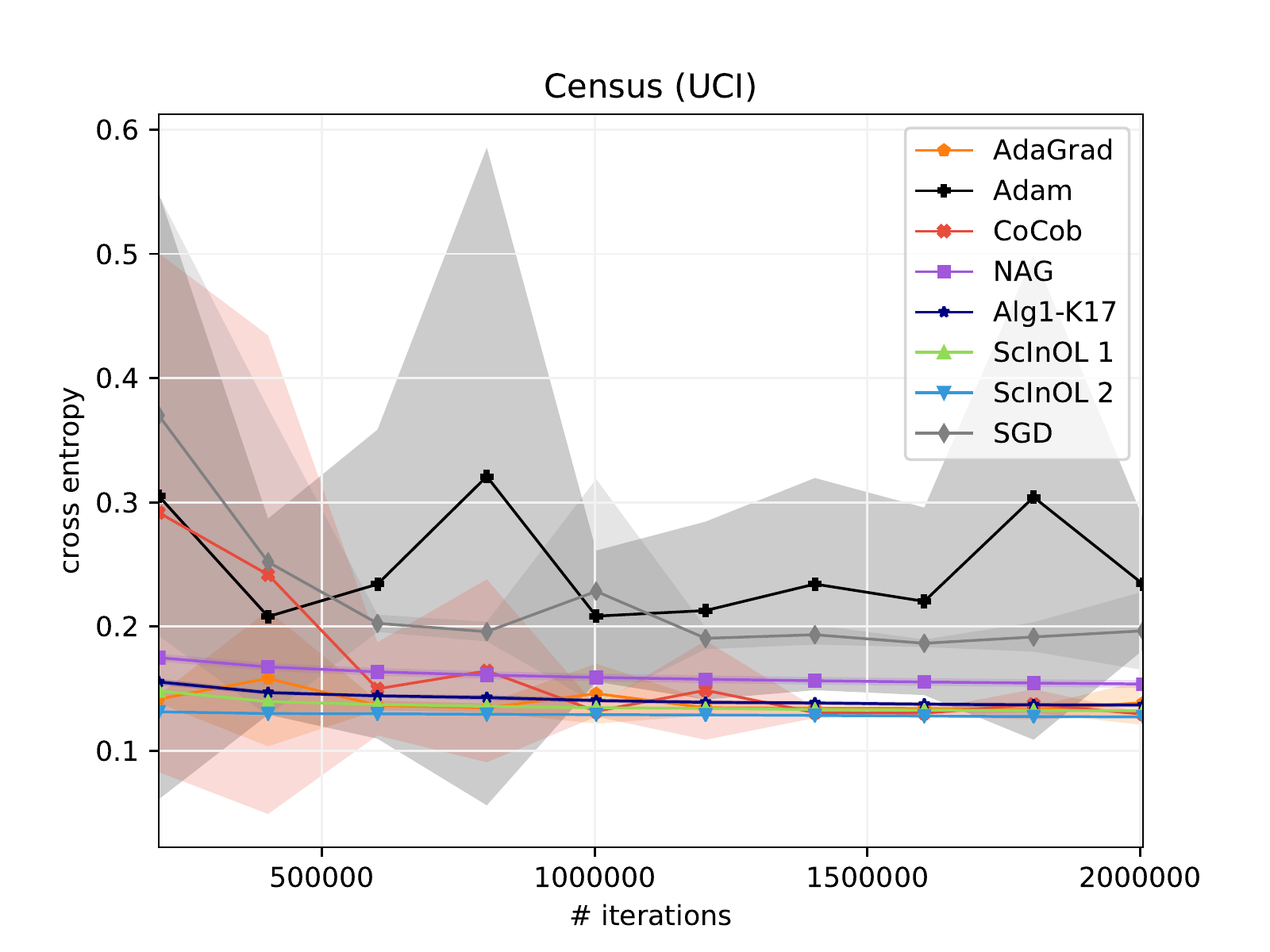}   \\
        \includegraphics[width=.33\textwidth]{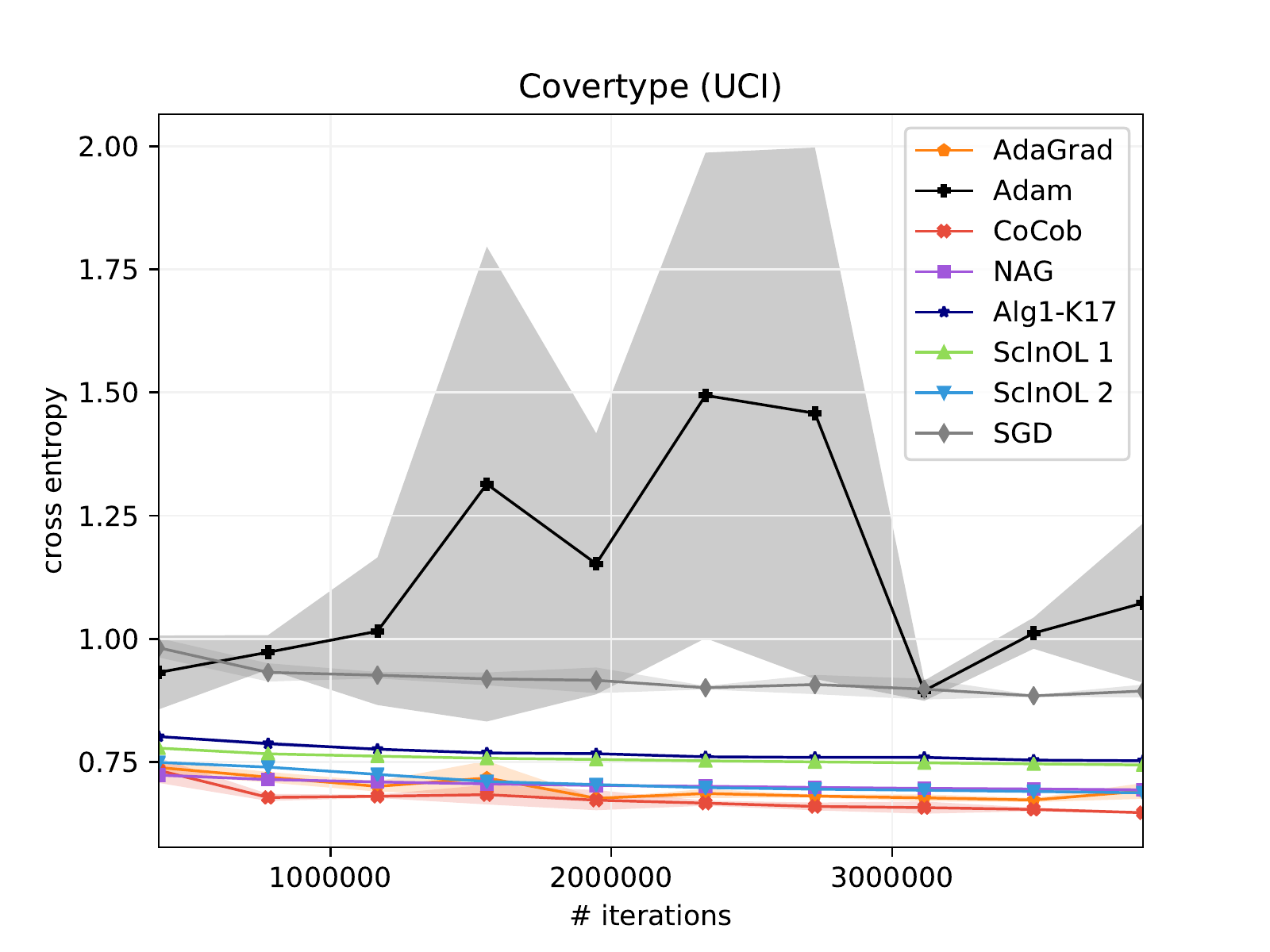} &
        \includegraphics[width=.33\textwidth]{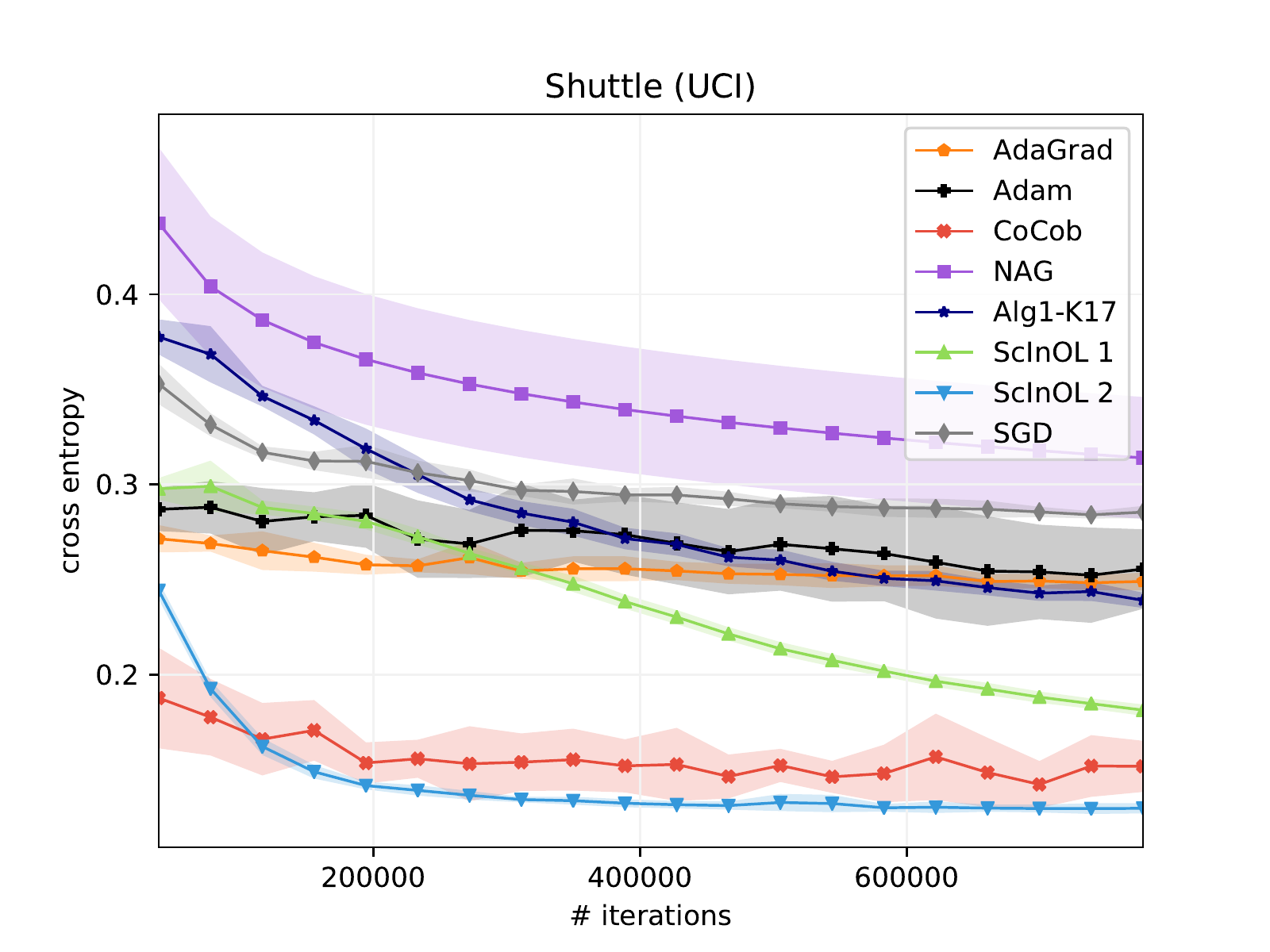} &
        \includegraphics[width=.33\textwidth]{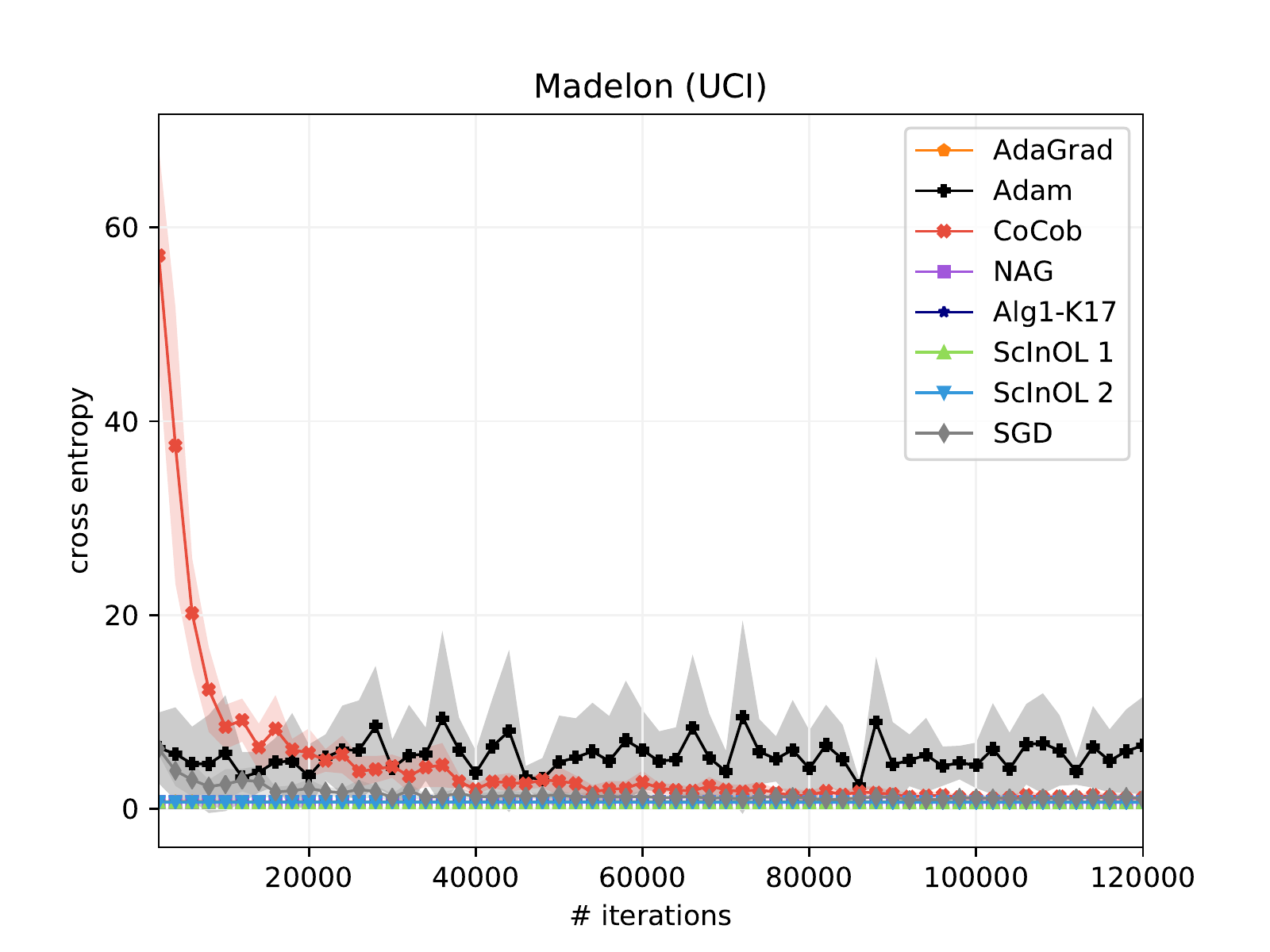}   \\
    \end{tabular}
    \caption{Mean test set cross entropy loss. Average values from 10 runs are depicted, shadows show $\pm$ standard deviation.}
    \label{fig:linear}
\end{figure*}

To further check empirical performance of our algorithms 
we tested them on some popular real-life benchmark datasets.
We chose 5 datasets with varying levels of feature scale variance from UCI repository \citep{Dua:2017} (Covertype, Census, Shuttle, Bank, Madelon) 
and a popular benchmark dataset MNIST \citep{lecun-mnisthandwrittendigit-2010}. For all datasets, categorical features were one-hot-encoded 
into multiple features and missing values were replaced by dedicated substitute features. 
Datasets that do not provide separate testing sets were split randomly into training/test sets ($2/1$ ratio). 
Short summary of datasets can be found in Appendix \ref{appx:datasets}.

Algorithms were trained by minimizing the cross entropy loss in an online fashion. Some of the data sets concern multiclass classification,
which is beyond the framework considered here, as it would require \emph{multivariate} prediction $\hby = (\hy_1,\ldots,\hy_K)$ for each of $K$ classes,
but it is straightforward to extend our setup to such multivariate case (details are given in Appendix \ref{appendix:multivariate_predictions}).
To gain more insight into the long-term behavior of the algorithms, we trained all algorithms for multiple epochs.
Each epoch consisted of running through the entire training set (shuffled) and testing average cross entropy and accuracy on the test set. 
Each algorithm was run 10 times for stability. We compared our algorithms with the following methods: 
SGD (with $\eta_t \sim 1/\sqrt{t}$), AdaGrad, Adam, NAG, CoCoB \citep{cocob} (an adaptive parameter-free algorithm; we used its Tensorflow implementation),
and Algorithm 1 (coordinate-wise scale invariant method) by \citet{Kotlowski_ALT2017} (which we call Alg1-K17).
The algorithms using hand-picked learning rates (SGD, AdaGrad, Adam, NAG) were run with values from $\{0.00001, 0.0001, 0.001, 0.01, 0.1, 1.0 \}$
(all other parameters were kept default), and only the best test set results were reported (note that this biases the results in favour of these algorithms).

Figure \ref{fig:linear} shows mean (test set) cross entropy loss (classification accuracy, given in Appendix \ref{appx:accuracy}, leads to essentially the same conclusions); 
shaded areas around each curve depicts $\pm$ one standard deviation (over different runs). 
All graphs start with the error measured after the first epoch for better readability. 
The most noticeable fact in the plots is comparatively high variance (between different runs) of SGD, AdaGrad, Adam and NAG, i.e. the approaches with tunable learning rate. They also often 
performed worse than the remaining algorithms. Among our methods, ScInOL$_2$ turned out to perform better than ScInOL$_1$ in every case, due to its more 
aggressive updates. Note, however, that both algorithms are surprisingly stable, exhibiting very small variance in their performance across different runs.
The best performance was most of the time achieved by either ScInOL$_2$ or CoCoB. Alg1-K17 was often converging somewhat slower (which is most pronounced for MNIST data), which
we believe is due to its very conservative update policy.

\section{Conclusions and future work}
We proposed two online algorithms which behavior 
is invariant under arbitrary rescaling of individual features. 
The algorithms do not require any prior knowledge on the scale of the instances or the comparator
and achieve, without any parameter tuning, regret bounds which match (up to a logarithmic factor) the regret bound of Online Gradient Descent
with optimally tuned separate learning rates per dimension. 
The algorithms run in $O(d)$ per trial, which is comparable
to the runtime of vanilla OGD. 

The framework considered in this paper concerns well-understood and relatively 
simple linear models with convex objectives. It would be interesting
to evaluate 
the importance of scale-invariance for deep learning methods,
comprised of multiple layers connected by non-linear activation functions. 
As scale-invariance leads to well-conditioned algorithms, 
we believe that it could not only
avoid the need for prior normalization of the inputs to the network, 
but it would also make the algorithm be independent of the 
scale of the inputs fed forward to the next layers. 
A scale-invariant update for neural nets
might be robust against the ``internal covariance shift'' phenomenon \citep{batchnorm} 
and avoid the need for batch normalization. 

Finally, the potential functions we use to analyze our updates seems
closely related to the potential function of EGU$^\pm$
\citep{eg}.
It may be that our tuned online updates are
simply approximation of (a version of) EGU$^\pm$
and this needs further investigation.

\clearpage

\section*{Acknowledgements}
M. Kempka and W. Kot{\l}owski were 
supported by the Polish National Science Centre under grant No. 2016/22/E/ST6/00299.
Part of this work was done while M.~K. Warmuth 
was at UC Santa Cruz, supported by NSF grant IIS-1619271.

%
%

\bibliography{bibliography}

\begin{thebibliography}{30}
\providecommand{\natexlab}[1]{#1}
\providecommand{\url}[1]{\texttt{#1}}
\expandafter\ifx\csname urlstyle\endcsname\relax
  \providecommand{\doi}[1]{doi: #1}\else
  \providecommand{\doi}{doi: \begingroup \urlstyle{rm}\Url}\fi

\bibitem[Azoury \& Warmuth(2001)Azoury and Warmuth]{aw}
Azoury, K.~S. and Warmuth, M.~K.
\newblock Relative loss bounds for on-line density estimation with the
  exponential family of distributions.
\newblock \emph{Machine Learning}, 43\penalty0 (3):\penalty0 211--246, 2001.

\bibitem[Boyd \& Vandenberghe(2004)Boyd and Vandenberghe]{BoydVandenberghe04}
Boyd, S. and Vandenberghe, L.
\newblock \emph{Convex Optimization}.
\newblock Cambridge University Press, 2004.

\bibitem[Cesa-Bianchi \& Lugosi(2006)Cesa-Bianchi and Lugosi]{PLGbook}
Cesa-Bianchi, N. and Lugosi, G.
\newblock \emph{Prediction, learning, and games}.
\newblock Cambridge University Press, 2006.

\bibitem[Cesa-Bianchi et~al.(1996)Cesa-Bianchi, Long, and
  Warmuth]{clw-wqlbop-96}
Cesa-Bianchi, N., Long, P., and Warmuth, M.~K.
\newblock Worst-case quadratic loss bounds for on-line prediction of linear
  functions by gradient descent.
\newblock \emph{IEEE Transactions on Neural Networks}, 7\penalty0 (2):\penalty0
  604--619, 1996.

\bibitem[Cutkosky \& Boahen(2017)Cutkosky and Boahen]{CutkoskyBoahen2017COLT}
Cutkosky, A. and Boahen, K.~A.
\newblock Online learning without prior information.
\newblock In \emph{Conference on Learning Theory ({COLT})}, pp.\  643--677,
  2017.

\bibitem[Cutkosky \& Orabona(2018)Cutkosky and
  Orabona]{CutkoskyOrabona2018COLT}
Cutkosky, A. and Orabona, F.
\newblock Black-box reductions for parameter-free online learning in banach
  spaces.
\newblock In \emph{Conference on Learning Theory ({COLT})}, pp.\  1493--1529,
  2018.

\bibitem[Dheeru \& Karra~Taniskidou(2017)Dheeru and Karra~Taniskidou]{Dua:2017}
Dheeru, D. and Karra~Taniskidou, E.
\newblock {UCI} machine learning repository, 2017.
\newblock URL \url{http://archive.ics.uci.edu/ml}.

\bibitem[Duchi et~al.(2011)Duchi, Hazan, and Singer]{adagrad}
Duchi, J.~C., Hazan, E., and Singer, Y.
\newblock Adaptive subgradient methods for online learning and stochastic
  optimization.
\newblock \emph{Journal of Machine Learning Research}, 12:\penalty0 2121--2159,
  2011.

\bibitem[Hazan(2015)]{Hazan_OCO}
Hazan, E.
\newblock Introduction to online convex optimization.
\newblock \emph{Foundations and Trends in Optimization}, 2\penalty0
  (3--4):\penalty0 157--325, 2015.

\bibitem[Ioffe \& Szegedy(2015)Ioffe and Szegedy]{batchnorm}
Ioffe, S. and Szegedy, C.
\newblock Batch normalization: Accelerating deep network training by reducing
  internal covariate shift.
\newblock In \emph{International Conference on Machine Learning ({ICML})}, pp.\
   448--456, 2015.

\bibitem[Kingma \& Ba(2014)Kingma and Ba]{DBLP:journals/corr/KingmaB14}
Kingma, D.~P. and Ba, J.
\newblock Adam: {A} method for stochastic optimization.
\newblock \emph{CoRR}, abs/1412.6980, 2014.
\newblock URL \url{http://arxiv.org/abs/1412.6980}.

\bibitem[Kivinen \& Warmuth(1997)Kivinen and Warmuth]{eg}
Kivinen, J. and Warmuth, M.~K.
\newblock Exponentiated gradient versus gradient descent for linear predictors.
\newblock \emph{Inf. Comput.}, 132\penalty0 (1):\penalty0 1--63, 1997.

\bibitem[Koren \& Livni(2017)Koren and Livni]{affine_invariant}
Koren, T. and Livni, R.
\newblock Affine-invariant online optimization and the low-rank experts
  problem.
\newblock In \emph{Advances in Neural Information Processing Systems 30}, pp.\
  4747--4755. Curran Associates, Inc., 2017.

\bibitem[Kot{\l}owski(2017)]{Kotlowski_ALT2017}
Kot{\l}owski, W.
\newblock Scale-invariant unconstrained online learning.
\newblock In \emph{Proceeding of the 28th International Conference on
  Algorithmic Learning Theory ({ALT} 2016)}, volume~76 of \emph{Proceedings of
  Machine Learning Research}, pp.\  412--433. PMLR, 2017.

\bibitem[LeCun \& Cortes(2010)LeCun and
  Cortes]{lecun-mnisthandwrittendigit-2010}
LeCun, Y. and Cortes, C.
\newblock {MNIST} handwritten digit database, 2010.
\newblock URL \url{http://yann.lecun.com/exdb/mnist/}.

\bibitem[Littlestone et~al.(1991)Littlestone, Long, and Warmuth]{llw-ollf-91}
Littlestone, N., Long, P.~M., and Warmuth, M.~K.
\newblock On-line learning of linear functions.
\newblock \emph{{ACM} Symposium on Theory of Computing ({STOC})}, pp.\
  465--475, 1991.

\bibitem[Luo et~al.(2016)Luo, Agarwal, Cesa-Bianchi, and
  Langford]{Luo_etal2016}
Luo, H., Agarwal, A., Cesa-Bianchi, N., and Langford, J.
\newblock Efficient second order online learning by sketching.
\newblock In \emph{Advances in Neural Information Processing Systems ({NIPS})
  29}, 2016.

\bibitem[McMahan \& Abernethy(2013)McMahan and Abernethy]{McMahanAbernethy2013}
McMahan, H.~B. and Abernethy, J.
\newblock Minimax optimal algorithms for unconstrained linear optimization.
\newblock In \emph{Advances in Neural Information Processing Systems ({NIPS})
  26}, pp.\  2724--2732, 2013.

\bibitem[McMahan \& Orabona(2014)McMahan and Orabona]{McMahanOrabona2014}
McMahan, H.~B. and Orabona, F.
\newblock Unconstrained online linear learning in {H}ilbert spaces: {M}inimax
  algorithms and normal approximation.
\newblock In \emph{Proc. of the 27th Conference on Learning Theory ({COLT})},
  pp.\  1020--1039, 2014.

\bibitem[Orabona(2013)]{Orabona2013}
Orabona, F.
\newblock Dimension-free exponentiated gradient.
\newblock In \emph{Advances in Neural Information Processing Systems ({NIPS})
  26}, pp.\  1806--1814, 2013.

\bibitem[Orabona(2014)]{Orabona2014}
Orabona, F.
\newblock Simultaneous model selection and optimization through parameter-free
  stochastic learning.
\newblock In \emph{Advances in Neural Information Processing Systems ({NIPS})
  27}, pp.\  1116--1124, 2014.

\bibitem[Orabona \& P{\'a}l(2015)Orabona and P{\'a}l]{OrabonaPal2015ALT}
Orabona, F. and P{\'a}l, D.
\newblock Scale-free algorithms for online linear optimization.
\newblock In \emph{Algorithmic Learning Theory ({ALT})}, pp.\  287--301, 2015.

\bibitem[Orabona \& P{\'a}l(2016)Orabona and P{\'a}l]{OrabonaPal2016NIPS}
Orabona, F. and P{\'a}l, D.
\newblock Coin betting and parameter-free online learning.
\newblock In \emph{Neural Information Processing Systems ({NIPS})}, 2016.

\bibitem[Orabona \& Tommasi(2017)Orabona and Tommasi]{cocob}
Orabona, F. and Tommasi, T.
\newblock Training deep networks without learning rates through coin betting.
\newblock In \emph{Advances in Neural Information Processing Systems ({NIPS})
  30}, pp.\  2157--2167, 2017.

\bibitem[Orabona et~al.(2015)Orabona, Crammer, and
  Cesa-Bianchi]{Orabona_etal2015}
Orabona, F., Crammer, K., and Cesa-Bianchi, N.
\newblock A generalized online mirror descent with applications to
  classification and regression.
\newblock \emph{Machine Learning}, 99\penalty0 (3):\penalty0 411--435, 2015.

\bibitem[Ross et~al.(2013)Ross, Mineiro, and Langford]{Ross_etal2013UAI}
Ross, S., Mineiro, P., and Langford, J.
\newblock Normalized online learning.
\newblock In \emph{Proc. of the 29th Conference on Uncertainty in Artificial
  Intelligence ({UAI})}, pp.\  537--545, 2013.

\bibitem[Shalev-Shwartz(2011)]{ShaiBook}
Shalev-Shwartz, S.
\newblock Online learning and online convex optimization.
\newblock \emph{Foundations and Trends in Machine Learning}, 4\penalty0
  (2):\penalty0 107--194, 2011.

\bibitem[Streeter \& McMahan(2012)Streeter and McMahan]{StreeterMcMahan2012}
Streeter, M. and McMahan, H.~B.
\newblock No-regret algorithms for unconstrained online convex optimization.
\newblock In \emph{Advances in Neural Information Processing Systems ({NIPS})
  25}, pp.\  2402--2410, 2012.

\bibitem[Vovk(2001)]{Vovk}
Vovk, V.
\newblock Competitive on-line statistics.
\newblock \emph{International Statistical Review}, 69\penalty0 (213-248), 2001.

\bibitem[Zinkevich(2003)]{gd}
Zinkevich, M.
\newblock Online convex programming and generalized infinitesimal gradient
  ascent.
\newblock In \emph{International Conference on Machine Learning ({ICML})}, pp.\
   928--936, 2003.

\end{thebibliography}
\bibliographystyle{icml2019}
\clearpage

\appendix
\onecolumn

\section{Bound for Online Gradient Descent with per-dimension learning rates}
\label{appendix:OGD}

We remind the update of OGD with per-dimension learning rates:
\[
w_{t+1,i} = w_{t,i} - \eta_i \nabla_{t,i}, \qquad i=1,\ldots, d,
\]
with $\w_1 = \boldsymbol{0}$.
For any $u_i \in \mathbb{R}$, we have:
\[
(u_i - w_{t+1,i})^2 - (u_i - w_{t,i})^2
= (u_i - w_{t,i} + \eta_i \nabla_{t,i})^2 - (u_i - w_{t,i})^2 
= 2 \eta_i \nabla_{t,i} (u_i - w_{t,i}) + \eta_i^2 \nabla_{t,i}^2.
\]
Summing over trials $t=1,\ldots,T$ and rearranging:
\[
2 \eta_i \sum_{t=1}^T \nabla_{t,i}(w_{t,i} - u_i)
= u_i^2 - (u_i - w_{T+1,i})^2 + \eta_i^2 \sum_{t=1}^T \nabla_{t,i}^2.
\]
Dividing by $2 \eta_i$, upper bounding and summing over $i=1,\ldots,d$:
\[
\sum_{t=1}^T \nabla_t^\top (\w_t - \u) 
\leq \sum_{i=1}^d \left(\frac{u_i^2 }{2 \eta_i} + \frac{\eta_i}{2} \sum_{t=1}^T \nabla_{t,i}^2 \right).
\]
Finally, using \eqref{eq:gradient_trick} shows that the right-hand side of the above upper bounds
the regret.

\section{Scale invariance of Algorithm \ref{alg:one} and Algorithm \ref{alg:two}}
\label{appendix:scale_invariance}

Let $\{(\x_t,y_t)\}_{t=1}^T$ be a data sequence and define
a transformed sequence $\{(\A \x_t,y_t)\}_{t=1}^T$, where 
$\A = \mathrm{diag}(a_1,\ldots,a_d)$ with $a_1,\ldots,a_d > 0$. We will
show that the sequence of predictions $\hy_1,\ldots,\hy_T$ generated
by the algorithms on the original and the transformed data sequences are the same.
This can easily be done inductively: assuming $\hy_1,\ldots,\hy_t$ are
the same on both sequences, this implies
$g_1,\ldots,g_t$ are also the same (as $g_t = \partial_{\hy_t} \ell(y_t,\hy_t)$,
while $y_t$ are the same in both sequences).
Given that, a closer inspection of the algorithms lets us determine
the behavior of all maintained statistics under the feature transformation
$x_{t,i} \mapsto a_i x_{t,i}$.

For both algorithms we have:
\[
M_{t,i} = \max_{j \leq t} |x_{j,i}|
\mapsto a_i M_{t,i}, \quad
S_{t,i}^2 = \sum_{j \leq t} (g_j x_{j,i})^2
\mapsto a_i^2 S_{t,i}^2, \quad
G_{t,i} = -\sum_{j \leq t} g_j x_{j,i}
\mapsto a_i G_{t,i},
\]
This means that for Algorithm \ref{alg:one}:
\[
  \beta_{t,i} \mapsto \beta_{t,i}, \quad
  \theta_{t,i} \mapsto \theta_{t,i}, \quad
  w_{t,i} \mapsto a_i^{-1} w_{t,i},
\]
so that $x_{t,i} w_{t,i} \mapsto x_{t,i} w_{t,i}$ and thus 
$\hy_t = \x_t^\top \w_t$ is invariant under the scale transformation.

Similarly, for Algorithm \ref{alg:two} we have:
\[
\eta_{t,i} \mapsto \eta_{t,i}, \quad \theta_{i,i} \mapsto \theta_{t,i}, \quad
w_{t,i} \mapsto a_i^{-1} w_{t,i},
\]
and the scale invariance follows.

\section{Proof of Theorem~\ref{thm:alg_one}}
\label{appx:alg_one}

Before proving the theorem, we need two auxiliary results:

\begin{lemma}
  Let $f(x) = \alpha \left( e^{|x| / \gamma} - |x| / \gamma - 1 \right)$
  with $\alpha, \gamma > 0$.
  Its Fenchel conjugate is given by:
  \begin{equation}
    \begin{split}
    f^*(u) ~\stackrel{def}{=}&~ \sup_{x} \{ux - f(x)\} \\
    ~=~ & (|u| \gamma + \alpha) \ln(1 + |u| \gamma / \alpha) - |u| \gamma \\
    ~\leq~ & |u| \gamma \ln(1 + |u| \gamma / \alpha).
    \end{split}
    \end{equation}

  \label{lem:conjugate_ScInOL_1}
\end{lemma}
\begin{proof}
  Note that since $f(x)$ is symmetric in $x$,
  \begin{equation}
    \begin{split}
    \sup_x \{ux - f(x)\} &= \sup_{x \geq 0} \{|u| x - f(x)\} \\
     &= \sup_{x \geq 0} \Big\{ \underbrace{
      |u| x - \alpha \left( e^{x / \gamma} - x/\gamma - 1 \right) }_{g(x)}
      \Big\}.
  \end{split}
    \end{equation}
  Setting the derivative of $g(x)$ to zero gives its unconstrained
  maximizer $x^* = \gamma \ln(1 + |u|\gamma/\alpha)$, and since $x^* \geq 0$, it is also
  the maximizer of $g(x)$ under constraint $x \geq 0$. Thus:
  \[
    f^*(u) = g(x^*) = (|u|\gamma + \alpha) \gamma \ln(1 + |u| \gamma / \alpha) - |u| \gamma.
  \]
  The inequality in the lemma follows from an elementary inequality $\ln(1+x) \leq x$
  applied to $\alpha \ln(1 + |u| \gamma/\alpha)$.
\end{proof}

\begin{lemma}
\label{lem:main_lemma_alg_one}
For any $v \in \mathbb{R}$ and any $q \in [-1,1]$:
\[
  \frac{q \; \mathrm{sgn}(v)}{2} \left(e^{\frac{|v|}{2}} - 1\right)
  + e^{\frac{|v-q|}{2\sqrt{1+q^2}}} - \frac{|v-q|}{2\sqrt{1+q^2}}
  \leq e^{\frac{|v|}{2}} - \frac{|v|}{2} + q^2.
\]
\end{lemma}
\begin{proof}
It suffices to prove the lemma for $v \geq 0$. Indeed, the inequality holds
for some $v \geq 0$ and $q \in [-1,1]$ if and only if it holds for $-v$ and $-q$.
Denote:
\[
  \vt = \frac{|v-q|}{\sqrt{1+q^2}}.
\]
In this notation and with the assumption $v \geq 0$, the inequality translates
to:
\begin{equation}
  e^{\frac{\vt}{2}} - \frac{\vt}{2} \leq
  e^{\frac{v}{2}}\left(1-\frac{q}{2}\right) - \frac{v - q}{2} + q^2
\label{eq:the_inequality}
\end{equation}
We will split the proof into three sub-cases: (i) $q \geq v$,
(ii) $q \leq v \leq 3$, and (iii) $v \geq 3$. Since $q \leq 1$, these cases
cover all allowed values of $v$ and $q$.

\paragraph{Case (i): $q \geq v$.} We have
$\vt = \frac{q-v}{\sqrt{1+q^2}} \leq q-v$. Since the function
$e^x - x$ is increasing in $x$ for $x\in(1,\infty)$, it holds:
\[
  e^{\frac{\vt}{2}} - \frac{\vt}{2}
  \leq e^{\frac{q-v}{2}} - \frac{q-v}{2}
  = e^{\frac{v}{2}}e^{\frac{q-2v}{2}} - \frac{q-v}{2}.
\]
From $q \leq 1$ and $v \geq 0$ it follows
$\frac{q-2v}{2} \leq  \frac{1-2v}{2} \leq \frac{1}{2}$.
Since function $f(x) = \frac{e^x - x - 1}{x^2}$ is nondecreasing in $x$
(see, e.g., \citep{PLGbook}, Section A.1.2), we have:
\begin{equation}
  e^x - x - 1 \leq x^2 \frac{e^{1/2} - 1/2 - 1}{1/4} \leq 0.6 x^2 \qquad
  \text{for~~} x \leq \frac{1}{2}.
\label{eq:exp_minus_x_minus_1_ineq}
\end{equation}
Thus, we bound $e^{\frac{q-2v}{2}}$
by $1 + \frac{q-2v}{2} + 0.15 (q-2v)^2$ and get:
\begin{align*}
e^{\frac{\vt}{2}} - \frac{\vt}{2}
&\leq e^{\frac{v}{2}} \left(1 + \frac{q-2v}{2} \right)
  - \frac{q-v}{2} + 0.15 e^{\frac{v}{2}} (q-2v)^2 \\
  &= e^{\frac{v}{2}} \left(1 - \frac{q}{2}\right)
  - \frac{v-q}{2} + (e^{\frac{v}{2}} - 1)(q-v)
 + 0.15 e^{\frac{v}{2}} (q-2v)^2 \\
&\leq e^{\frac{v}{2}} \left(1 - \frac{q}{2}\right)
 - \frac{v-q}{2} + v(q-v)
  + \frac{1}{4} (q-2v)^2,
\end{align*}
where the last inequality follows from the fact that $v \leq 1$ (as $q \geq v$ and
$q \leq 1$), which by \eqref{eq:exp_minus_x_minus_1_ineq} implies
$e^{\frac{v}{2}} \leq 1 + \frac{v}{2} + 0.6 \frac{v^2}{4} = 1 + 0.5v + 0.15v^2
\leq 1 + v$, and furthermore $0.15 e^{\frac{v}{2}} \leq 0.15 e^{\frac{1}{2}} \leq \frac{1}{4}$.
But $v(q-v) + \frac{1}{4} (q-2v)^2 = \frac{1}{4}q^2 \leq q^2$, which
proves \eqref{eq:the_inequality} for $q \geq v$.

\paragraph{Case (ii): $q \leq v \leq 3$.}
We have $\vt = \frac{v-q}{\sqrt{1+q^2}} \leq v-q$, and by the monotonicity
of function $e^x - x$ for $x\in(1,\infty)$:
\[
  e^{\frac{\vt}{2}} - \frac{\vt}{2}
  \leq e^{\frac{v-q}{2}} - \frac{v-q}{2}
  = e^{\frac{v}{2}} e^{-\frac{q}{2}} - \frac{v-q}{2}.
\]
Using \eqref{eq:exp_minus_x_minus_1_ineq} and $q \geq -1$, we bound
$e^{-q/2} \leq 1 - \frac{q}{2}+ 0.15 q^2$ to get:
\[
  e^{\frac{\vt}{2}} - \frac{\vt}{2}
  \leq e^{\frac{v}{2}} \left(1 - \frac{q}{2}\right) - \frac{v-q}{2}
   + 0.15 e^{\frac{v}{2}} q^2.
\]
Using $0.15 e^{\frac{v}{2}} \leq 0.15 e^{\frac{3}{2}} \leq 0.68 \leq 1$
proves \eqref{eq:the_inequality} for $q \leq v \leq 3$.

\paragraph{Case (iii): $v > 3$.} We lower-bound the right-hand side
of \eqref{eq:the_inequality}:
\[
e^{\frac{v}{2}}\left(1-\frac{q}{2}\right) - \frac{v - q}{2} + q^2
\geq e^{\frac{v}{2}}\left(1-\frac{q}{2}\right) - \frac{v - q - \frac{q^2}{2}}{2}
\geq e^{\frac{1}{2}(v- q- \frac{q^2}{2})}- \frac{v - q - \frac{q^2}{2}}{2},
\]
where the first inequality is simply from $q^2 \geq \frac{q^2}{4}$, while
the second follows from $1 - x \geq e^{-x-x^2}$ for $x \leq \frac{1}{2}$
(see, .e.g., \citep{PLGbook}, Lemma 2.4). Now, using the monotonicity
of function $e^x - x$,
\[
e^{\frac{1}{2}(v- q- \frac{q^2}{2})}- \frac{v - q - \frac{q^2}{2}}{2}
\geq e^{\frac{\vt}{2}} - \frac{\vt}{2}
\qquad
\iff \qquad
v- q- \frac{q^2}{2} \geq \vt,
\]
thus it suffices to show the latter to finish the proof.
We have:
\[
  v-q-\frac{q^2}{2} - \vt
  = (v-q)\left(1 - \frac{1}{\sqrt{1+q^2}}\right) - \frac{q^2}{2}
  \geq (3-q)\left(1 - \frac{1}{\sqrt{1+q^2}}\right) - \frac{q^2}{2}.
\]
Using elementary inequality $\sqrt{1+x} \leq 1 + \frac{x}{2}$, we have:
$\frac{1}{\sqrt{1+q^2}} = \frac{\sqrt{1+q^2}}{1+q^2} \leq \frac{1 + q^2/2}{1+q^2}$, and thus:
\begin{align*}
v-q-\frac{q^2}{2} - \vt
&\geq (3-q)\left(1 - \frac{1 + q^2/2}{1+q^2}\right) - \frac{q^2}{2}
= (3-q)\frac{q^2/2}{1+q^2} - \frac{q^2}{2} \\
&= \frac{q^2}{2} \left( \frac{3-q}{1 + q^2} - 1\right)
\geq \frac{q^2}{2} \left( \frac{3-1}{1 + 1} - 1\right) = 0.
\end{align*}
This shows that $v- q- \frac{q^2}{2} \geq \vt$ and
thus proves \eqref{eq:exp_minus_x_minus_1_ineq} for $v > 3$.
\end{proof}

Before we state the next result, we summarize the notation which will be used
in what follows. For any $i=1,\ldots,d$ and
any $t=1,\ldots,T$, let:
  \[
    M_{t,i} = \max_{j \leq t} |x_{j,i}|, \qquad
    G_{t,i} = -\sum_{j \leq t} g_j x_{j,i}, \qquad
    S^2_{t,i} = \sum_{j \leq t} (g_j x_{j,i})^2,
  \]
  be, respectively, the maximum input value, the negative cumulative
  gradient, and the sum of squared gradients at $i$-th coordinate up to (and including)
  trial $t$, and we also denote $M_{0,i}=G_{0,i}=S^2_{0,i}=0$. Moreover, define:
  \[
    \beta_{t,i} = \left\{
      \begin{array}{ll}
      \min\left\{\beta_{t-1,i}, \epsilon \frac{S_{t-1,i}^2
        + M_{t,i}^2}{x_{t,i}^2 t}\right\}
        & \quad \text{when~~} x_{t,i} \neq 0, \\
      \beta_{t-1,i} & \quad \text{when~~} x_{t,i} = 0,
      \end{array}
    \right.
  \]
  with $\beta_{1,i} = \epsilon$.
  The weight vector at trial $t$ is given by:
  \begin{equation}
    w_{t,i} = \frac{\beta_{t,i} \mathrm{sgn}(G_{t-1,i})}{2 \sqrt{S_{t-1,i}^2 + M_{t,i}^2}}
    \left(e^{\frac{|G_{t-1,i}|}{2\sqrt{S_{t-1,i}^2 + M_{t,i}^2}}} - 1 \right),
    \label{eq:ScInOL_1_weights}
  \end{equation}
  as long as $M_{t,i} > 0$; if $M_{t,i} = 0$ (which means that $x_{j,i} = 0$
  for all $j \leq t$), we set $w_{t,i}=0$, but any other value of $w_{t,i}$ would
  lead to the same loss. Finally, define $\hat{S}^2_{t,i} = S^2_{t,i} + M^2_{t,i}$.

\begin{lemma}
  Define:
  \[
    \psi_{t,i}(x) =
    \left\{
      \begin{array}{ll}
      \beta_{t,i} \left(e^{|x| / (2 \hat{S}_{t,i})} - \frac{|x|}{2 \hat{S}_{t,i}} - 1
      \right) &\quad \text{for~~} \hat{S}_{t,i} \neq 0, \\
      0 &\quad \text{for~~} \hat{S}_{t,i} = 0.
     \end{array}
    \right.
  \]
  For any $i=1,\ldots,d$ and any $t=1,\ldots,T$ we have:
  \[
    w_{t,i} g_t x_{t,i} ~\leq~ \psi_{t-1,i}(G_{t-1,i}) - \psi_{t,i}(G_{t,i}) +
    \frac{\epsilon}{t}.
  \]
  \label{lem:ScInOL_1_progress}
\end{lemma}
\begin{proof}
  Fix $i \in \{1,\ldots,d\}$, and
  let $\tau_i$ be the first trial $t$ such that $x_{t,i} \neq 0$. This means
  that $\hat{S}_{t,i} = x_{t,i} = 0$ for all $t < \tau_i$, and the inequality
  is trivially satisfied for any $t < \tau_i$, as the left-hand side is zero, while
  the right-hand side is $\epsilon/t$. Thus, assume $t \geq \tau_i$.

  Fix $t$ and define $v = \frac{G_{t-1,i}}{\sqrt{S_{t-1,i}^2 + M_{t,i}^2}}$ and
  $q = \frac{g_t x_{t,i}}{\sqrt{S_{t-1,i}^2 + M_{t,i}^2}}$. As
  $|q| \leq \frac{|g_t x_{t,i}|}{M_{t,i}} \leq \frac{|x_{t,i}|}{\max_{j \leq t}
  |x_{j,i}|} \leq 1$, we can apply Lemma \ref{lem:main_lemma_alg_one} to such
  $v$ and $q$, which, after subtracting $1$ and multiplying by $\beta_{t,i}$
  on both sides, gives:
  \begin{align}
  \beta_{t,i}
  \frac{q \; \mathrm{sgn}(v)}{2} \left(e^{\frac{|v|}{2}} - 1\right)
    &+ \beta_{t,i}
  \left(e^{\frac{|v-q|}{2\sqrt{1+q^2}}} - \frac{|v-q|}{2\sqrt{1+q^2}} - 1\right)
    \nonumber \\
    &\leq \beta_{t,i}
  \left(e^{\frac{|v|}{2}} - \frac{|v|}{2} - 1 \right) + \beta_{t,i} q^2.
    \label{eq:after_applying_lemma}
  \end{align}
  Using the definition of the weight vector \eqref{eq:ScInOL_1_weights} we identify
  the first term on the left-hand side of \eqref{eq:after_applying_lemma}:
  \[
    \beta_{t,i}\frac{ q \; \mathrm{sgn}(v)}{2} \left(e^{|v|/2} - 1 \right)
    = w_{t,i} g_t x_{t,i}.
  \]
  Next, since:
  \[
    \frac{G_{t,i}}{\hat{S}_{t,i}} =
    \frac{G_{t,i}}{\sqrt{S_{t,i}^2 + M_{t,i}^2}}
    = \frac{G_{t-1,i} - g_t x_{t,i}}{
      \sqrt{S_{t-1,i}^2 + M_{t,i}^2 + (g_t x_{t,i})^2}}
    = \frac{v-q}{\sqrt{1+q^2}},
  \]
  the second term on the left-hand side of \eqref{eq:after_applying_lemma} is
  equal to $\psi_{t,i}(G_{t,i})$. Thus, \eqref{eq:after_applying_lemma} can be
  rewritten as:
  \[
    w_{t,i} g_t x_{t,i} + \psi_{t,i}(G_{t,i}) \leq
  \beta_{t,i}
    \left(e^{\frac{|v|}{2}} - \frac{|v|}{2} - 1 \right) + \beta_{t,i} q^2,
  \]
  and to finish the proof, it suffices to show that the two terms on the
  right-hand side
  are upper bounded, respectively, by $\psi_{t-1,i}(G_{t-1,i})$ and $\frac{\epsilon}{t}$.

  To bound $\beta_{t,i} q^2$ note that if $x_{t,i}=0$ then
  $\beta_{t,i} q^2=0$, whereas if $x_{t,i} \neq 0$ then
  by the definition of $\beta_{t,i}$:
  \[
  \beta_{t,i} q^2 = \beta_{t,i} \frac{(g_t x_{t,i})^2}{S_{t-1,i}^2 + M_{t,i}^2}
  \leq \epsilon \frac{S_{t-1,i}^2 + M_{t,i}^2}{x_{t,i}^2 t}
  \frac{(g_t x_{t,i})^2}{S_{t-1,i}^2 + M_{t,i}^2}
    \leq \frac{\epsilon g_t^2}{t} \leq \frac{\epsilon}{t}.
  \]
  To bound $\beta_{t,i}(e^{|v|/2} - |v|/2 - 1)$ by $\psi_{\tau_i-1,i}(G_{t-1,i})$
  note that both are zero if $t=\tau_i$ (because $G_{\tau_i-1,i}=0$ and $v = 0$).
  On the other hand, for $t > \tau_i$ we have:
  \[
   |v| = \frac{|G_{t-1,i}|}{\sqrt{S_{t-1,i}^2 + M_{t,i}^2}}
   \leq \frac{|G_{t-1,i}|}{\sqrt{S_{t-1,i}^2 + M_{t-1,i}^2}}
   = \frac{|G_{t-1,i}|}{\hat{S}_{t-1,i}},
  \]
  and by the monotonicity of $f(x) = e^x - x -1$:
  \[
  \beta_{t,i}(e^{|v|/2} - |v|/2 - 1)
  \leq \beta_{t,i} \left(e^{\frac{|G_{t-1,i}|}{2 \hat{S}_{t-1,i}}}
  - \frac{|G_{t-1,i}|}{2 \hat{S}_{t-1,i}}- 1 \right)
  \leq \psi_{t-1,i} (G_{t-1,i}),
  \]
  where in the last inequality we used $\beta_{t,i} \leq \beta_{t-1,i}$
  (which follows from the definition) and the fact that $e^x - x -1 \geq 0$ for all $x$.
\end{proof}

We are now ready to prove Theorem \ref{thm:alg_one}, which we restate here for
convenience:

\begin{theorem*}
For any $\u \in \mathbb{R}$ the regret of $\mathrm{ScInOL}_1$ is upper-bounded by:
\[
  \regret_T(\u)
  \leq \sum_{i=1}^d \left(2|u_i| \hat{S}_{T,i} \ln (1 + 2|u_i| \hat{S}_{T,i} \epsilon^{-1} T) + \epsilon(1 + \ln T )\right) = \sum_{i=1}^d \tilde{O}(|u_i| \hat{S}_{T,i}),
\]
where $\hat{S}_{T,i} = \sqrt{S_{T,i}^2 + M_{T,i}^2}$ and $\tilde{O}(\cdot)$ hides the
constants and logarithmic factors.
\end{theorem*}

\begin{proof}
  Applying Lemma \eqref{lem:ScInOL_1_progress} for a fixed $i \in \{1,\ldots,d\}$
  and all $t=1,\ldots,T$, and summing over trials gives:
  \[
    \sum_{t=1}^T w_{t,i} g_t x_{t,i} \leq - \psi_{T,i}(G_{T,i}) +
    \sum_{t=1}^T \frac{\epsilon}{t}
   \leq - \psi_{T,i}(G_{T,i}) + \epsilon\left(1 + \ln T \right),
  \]
  where we used $\psi_{0,i}(G_{0,i}) = 0$.
  By \eqref{eq:gradient_trick},
  \begin{align*}
    \regret_T(\u) &\leq \sum_{t=1}^T g_t \x_t^\top (\w_t - \u)
    =  \sum_{i=1}^d \left( \sum_{t=1}^T g_t x_{t,i}w_{t,i} +  G_{T,i} u_i \right) \\
    &\leq
    \sum_{i=1}^d \left( G_{T,i} u_i - \psi_{T,i}(G_{T,i}) \right)
    + d \epsilon \left(1 + \ln(T) \right) \\
    &\leq
    \sum_{i=1}^d \sup_{x} \left\{ x u_i -
    \psi_{T,i}(x) \right\}
    + d \epsilon \left(1 + \ln(T) \right) \\
    &\leq
    \sum_{i=1}^d 2 |u_i| \hat{S}_{T,i} \ln\left(1 + 2 |u_i| \hat{S}_{T,i} / \beta_{T,i}
    \right)
     + d \epsilon \left(1 + \ln(T) \right),
  \end{align*}
  where in the last inequality we used Lemma \ref{lem:conjugate_ScInOL_1}
  for each $i$
  with $\alpha = \beta_{T,i}$ and $\gamma = 2 \hat{S}_{T,i}$. To finish
  the proof, it suffices to show that $\beta_{T,i} \geq \frac{\epsilon}{T}$, which
  we do by induction on $t$. For $t=1$, we have by the definition
  $\beta_{t,i} = \epsilon$. Now, assume $\beta_{t-1,i} \geq \frac{\epsilon}{t-1}$,
  and we will show $\beta_{t,i} \geq \frac{\epsilon}{t}$.
  If $x_{t,i} = 0$, $\beta_{t,i} = \beta_{t-1,i} \geq \frac{\epsilon}{t-1}
  > \frac{\epsilon}{t}$; on the other hand, if $x_{t,i} \neq 0$, from the definition
  of $\beta_{t,i}$:
  \[
  \beta_{t,i} = \min\left\{\beta_{t-1,i}, \epsilon \frac{S_{t-1,i}^2
        + M_{t,i}^2}{x_{t,i}^2 t}\right\}
        \geq \min \left\{\frac{\epsilon}{t-1}, \epsilon \frac{x_{t,i}^2}{x_{t,i}^2 t} \right\}
        = \frac{\epsilon}{t},
  \]
    where we used $S_{t-1,i}^2 + M_{t,i}^2 \geq M_{t,i}^2 = \max_{j \leq t} x_{j,i}^2
    \geq x_{t,i}^2$.
\end{proof}

\section{Proof of Theorem~\ref{thm:alg_two}}
        \label{appx:alg_two}

        \begin{figure}[t]
        {\centering
        \begin{tikzpicture}[font=\scriptsize]
            \begin{axis}[
                height=6cm,
                width=8cm,
                grid=major,
            xlabel={$x$},
            ylabel={$y$},
            xmin=-2, xmax=2, ymin=-0.4, ymax=2,
            samples=50,
            legend cell align={left},
            smooth,
            legend style={at={(0.5,0.98)},anchor=north},
            ]
            \addplot[densely dashed,thick] {0.5 * x^2};
            \addlegendentry{$y=\frac{1}{2}x^2$}
            \addplot[densely dotted,thick] {sqrt(x^2) - 0.5};
            \addlegendentry{$y=|x| - \frac{1}{2}$}
            \addplot[black,thick] {ifthenelse(abs(x)<1, 0.5*x^2, abs(x) - 0.5)};
            \addlegendentry{$y=h(x)$}
            \end{axis}
        \end{tikzpicture}
        \par}
        \caption{Function $h(x)$}
          \label{fig:h}
        \end{figure}
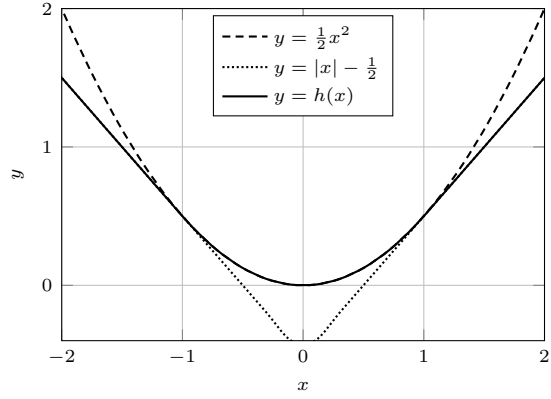

        Similarly as in the previous section, we proceed the proof of the theorem with
        several auxiliary results. Define:
        \begin{equation}
          h(x) = \left\{
            \begin{array}{ll}
              \frac{1}{2} x^2 & \qquad \text{for~} |x| \leq 1, \\
              |x| - \frac{1}{2} & \qquad \text{for~} |x| > 1
            \end{array}
            \right.
        \label{eq:h_function}
        \end{equation}
        (see Figure \ref{fig:h}).
        Note that $h(x) = h(|x|)$, and $h(|x|)$ is monotonic in $|x|$.
        Moreover, for all $x \in \mathbb{R}$:
        \begin{equation}
         |x| - \frac{1}{2} \leq h(x) \leq \frac{1}{2} x^2.
          \label{eq:bounds_on_h_x}
        \end{equation}
        The lower bound in \eqref{eq:bounds_on_h_x} is clearly satisfied for $|x| < 1$, while
        for $|x| \leq 1$ we have
        $h(x) - (|x| - \frac{1}{2}) = \frac{1}{2}(|x| - 1)^2 \geq 0$.
        On the other hand, the upper bound in \eqref{eq:bounds_on_h_x} is clearly
        satisfied for $|x| \leq 1$, while for $|x| > 1$
        we have $h(x) - \frac{1}{2} x^2 = -\frac{1}{2}(|x| - 1)^2 \leq 0$.

        \begin{lemma}
          Let $f(x) = \alpha e^{|x| / \gamma}$ with $\alpha, \gamma > 0$.
          Its Fenchel conjugate $f^*(u) = \sup_x \{ux - f(x)\}$
          satisfies $f^*(u) \leq |u| \gamma (\ln(|u|\gamma / \alpha) - 1)$ for all $u$.
          \label{lem:conjugate_ScInOL_2}
        \end{lemma}
        \begin{proof}
          Since $f(x)$ is symmetric in $x$,
            $\sup_x \{ux - f(x)\} = \sup_{x \geq 0} \{|u| x - f(x)\} = \sup_{x \geq 0 } g(x)$,
            where $g(x) = |u| x - \alpha e^{x / \gamma}$.
          Setting the derivative of $g(x)$ to zero gives its unconstrained
          maximizer $x^* = \gamma \ln(|u|\gamma/\alpha)$, for which
          $g(x^*) = |u| \gamma (\ln(|u|\gamma / \alpha) - 1)$. The proof is finished by
          noticing that $\sup_{x \geq 0} g(x) \leq \sup_{x \in \mathbb{R}} g(x) = g(x^*)$.
        \end{proof}

        \begin{lemma}
        \label{lem:main_lemma_alg_two}
        For any $v \in \mathbb{R}$ and any $q \in [-1,1]$:
        \[
          \exp\left\{\frac{1}{2}h\left(\frac{v-q}{1+q^2}\right)
          - \frac{1}{2} h(v) - \frac{1}{2} q^2 \right\} \leq 1
          - \frac{1}{2} q \; \mathrm{sgn}(v) \min\{|v|,1\}
        \]
        \end{lemma}
        \begin{proof}
        It suffices to prove the lemma for $v \geq 0$. Indeed, the inequality holds
        for some $v \geq 0$ and $q \in [-1,1]$ if and only if it holds for $-v$ and $-q$.
        Denote:
        \[
          \vt = \frac{|v-q|}{\sqrt{1+q^2}}.
        \]
        In this notation and with the assumption $v \geq 0$, the inequality translates
        to:
          \begin{equation}
            e^{\frac{1}{2}(h(\vt) - h(v) - q^2)} \leq 1 - \frac{1}{2} q
            \min\{v,1\}.
            \label{eq:lemma_inequality_translated}
          \end{equation}
          To prove \eqref{eq:lemma_inequality_translated}, it suffices
          to show that:
          \begin{equation}
            h(\vt) - h(v) - q^2 \leq - q \min\{v,1\} - \frac{1}{2} (q \min\{v,1\})^2
            \label{eq:to_prove_main_lemma_two},
          \end{equation}
          because \eqref{eq:to_prove_main_lemma_two} together with $q \leq 1$ and
          inequality $e^{-x -x^2} \leq 1 - x$ for $x \leq \frac{1}{2}$
          (see, e.g., \citep{PLGbook}, Section A.1.2) implies \eqref{eq:lemma_inequality_translated}.

          We will split the proof of \eqref{eq:to_prove_main_lemma_two}
          into three sub-cases: (i) $v \leq 1$,
        (ii) $v \geq 1$ and $\vt \geq 1$, (iii) $v \geq 1$ and $\vt < 1$.

        \paragraph{Case (i): $v \leq 1$.}
        From the definition, $h(v) = \frac{1}{2} v^2$ and by \eqref{eq:bounds_on_h_x}
        we upper bound $h(\vt) \leq \frac{1}{2} \vt^2$. Using $\vt \leq |v-q|$ we have:
        \[
        h(\vt)-h(v)-q^2
        \leq \frac{1}{2} \vt^2 - \frac{1}{2}v^2 - q^2
        \leq \frac{1}{2}(v-q)^2 - \frac{1}{2}v^2 - q^2
        = -vq - \frac{1}{2}q^2
          \leq -vq - \frac{1}{2} v^2 q^2,
        \]
          and since $\min\{v,1\} = v$, this implies \eqref{eq:to_prove_main_lemma_two}.

        \paragraph{Case (ii): $v \geq 1$ and $\vt \geq 1$.}
        As $q \leq 1 \leq v$, we have $|v-q| = v-q$, and by the definition,
          $h(v) = v - \frac{1}{2}$,
          $h(\vt) = \vt - \frac{1}{2}$. Therefore:
        \[
        h(\vt)-h(v)-q^2 = \vt - v -q^2
        \leq v-q -v -q^2  \leq -q - q^2 / 2,
        \]
        where in the first inequality we used $\vt \leq |v-q| = v-q$.
          As $\min\{v,1\} = 1$, this implies \eqref{eq:to_prove_main_lemma_two}.

        \paragraph{Case (iii): $v \geq 1$ and $\vt < 1$.}
        We have:
        \[
         \vt < 1 \; \iff \;
         \frac{(v-q)^2}{1+q^2} \leq 1 \; \iff \;
         v^2 - 2vq - 1 \leq 0 \; \iff \; v \leq q + \sqrt{1+q^2},
        \]
        where the last equivalence follows from solving a quadratic inequality with respect
        to $v \geq 1$ for fixed $q$.
        We now note that function:
        \[
          g(v) = h(\vt) - h(v) - q^2 = \frac{1}{2} \vt^2 - \Big(v - \frac{1}{2}\Big) - q^2
          = \frac{(v-q)^2}{2(1+q^2)} - v - q^2 + \frac{1}{2}
        \]
        is convex in $v$ and hence it is maximized at the boundaries $\{1,q + \sqrt{1+q^2}\}$
        of the allowed range of $v$.
        When $v=1$, we have:
        \[
          g(v) = \frac{(1-q)^2}{2(1+q^2)} - 1 - q^2 + \frac{1}{2}
        \leq \frac{1}{2}(1-q)^2 - q^2 - \frac{1}{2}
        = -q - \frac{1}{2}q^2,
        \]
        whereas if $v = q + \sqrt{1+q^2}$, we have
        \[
          g(v) = \frac{1}{2} - \left(q + \sqrt{1+q^2}\right) - q^2 + \frac{1}{2}
          \leq -q -q^2 \leq -q  - \frac{1}{2} q^2,
        \]
        so that $g(v) \leq -q - \frac{1}{2} q^2$ in the entire range of allowed values of $v$.
          As $\min\{v,1\} = 1$, this implies \eqref{eq:to_prove_main_lemma_two}.
        \end{proof}

        Before stating further results, we summarize the notation: for $i=1,\ldots,d$
        and $t=1,\ldots,T$,
          \[
            M_{t,i} = \max_{j \leq t} |x_{j,i}|, \quad
            G_{t,i} = -\sum_{j \leq t} g_j x_{j,i}, \quad
            S^2_{t,i} = \sum_{j \leq t} (g_j x_{j,i})^2, \quad
            \eta_{t,i} = \epsilon - \sum_{j \leq t} g_t x_{t,i} w_{t,i},
          \]
          with the convention $M_{0,i}=G_{0,i}=S^2_{0,i}=0$ and $\eta_{0,i} = \epsilon$.
          As before, we also use $\hat{S}^2_{t,i} = S^2_{t,i} + M^2_{t,i}$.
          The weight vector at trial $t$ is given by:
          \begin{equation}
            w_{t,i} = \frac{\mathrm{sgn}(G_{t-1,i}) \min\left\{\frac{|G_{t-1,i}|}{
              \sqrt{S_{t-1,i}^2 + M_{t,i}^2}}, 1\right\}}{2 \sqrt{S_{t-1,i}^2 + M_{t,i}^2}}
            \eta_{t-1,i}
            \label{eq:ScInOL_2_weights}
          \end{equation}
          as long as $M_{t,i} > 0$; if $M_{t,i} = 0$,
          we set $w_{t,i}=0$.

        \begin{lemma}
          Define:
          \[
            \psi_{t,i}(x) =
            \left\{
              \begin{array}{ll}
                e^{\frac{1}{2} h \big(\frac{x}{\hat{S}_{t,i}}\big)} &\quad
                \text{for~~} \hat{S}_{t,i} \neq 0, \\
              1 &\quad \text{for~~} \hat{S}_{t,i} = 0,
             \end{array}
            \right.
          \]
          with $h(\cdot)$ defined in \eqref{eq:h_function}.
          For any $i=1,\ldots,d$, let $\tau_i$ be the first
          trial in which $x_{t,i} \neq 0$.
          We have for any $=1,\ldots,d$ and any $t=\tau_i,\ldots,T$:
          \[
            \frac{\eta_{t,i}}{\eta_{t-1,i}} \geq \frac{\psi_{t,i}(G_{t,i})}
            {\psi_{t-1,i}(G_{t-1,i})} e^{-\delta_{t,i}},
          \]
          where $\delta_{t,i} = \frac{(g_t x_{t,i})^2}{2 (S_{t-1,i}^2 + M_{t,i}^2)}$
          \label{lem:ScInOL_2_progress}
        \end{lemma}
        \begin{proof}

          Fix $i$ and $t \geq \tau_i$,
          and define $v = \frac{G_{t-1,i}}{\sqrt{S_{t-1,i}^2 + M_{t,i}^2}}$ and
          $q = \frac{g_t x_{t,i}}{\sqrt{S_{t-1,i}^2 + M_{t,i}^2}}$. As
          $|q| \leq \frac{|g_t x_{t,i}|}{M_{t,i}} \leq 1$,
          we can apply Lemma \ref{lem:main_lemma_alg_two} to such
          $v$ and $q$, which gives:
          \begin{equation}
          e^{\frac{1}{2}h\left(\frac{v-q}{1+q^2}\right)
          - \frac{1}{2} h(v) - \frac{1}{2} q^2 } \leq 1
          - \frac{1}{2} q \; \mathrm{sgn}(v) \min\{|v|,1\}
          \label{eq:intermediate_after_lemma}
          \end{equation}
          Using the definition of weight vector \eqref{eq:ScInOL_2_weights}, we identify
          the right-hand side of \eqref{eq:intermediate_after_lemma} with
          $1 - \frac{g_t x_{t,i} w_{t,i}}{\eta_{t-1,i}} = \frac{\eta_{t,i}}{\eta_{t-1,i}}$.
          Since $\frac{1}{2} q^2 = \delta_{t,i}$ and
          $\frac{G_{t,i}}{\hat{S}_{t,i}} = \frac{v-q}{\sqrt{1+q^2}}$ (see the proof
          of Lemma \ref{lem:ScInOL_1_progress}),
          we also identify the left-hand side of \eqref{eq:intermediate_after_lemma}
          with $\psi_{t,i}(G_{t,i}) e^{-\frac{1}{2} h(v)} e^{-\delta_{t,i}}$. Hence,
          \eqref{eq:intermediate_after_lemma} can be rewritten as:
          \[
          \frac{\eta_{t,i}}{\eta_{t-1,i}}
           \geq \frac{\psi_{t,i}(G_{t,i})}{e^{\frac{1}{2} h(v)}} e^{-\delta_{t,i}},
          \]
          and thus to prove the lemma, it suffices to show:
          \begin{equation}
          e^{\frac{1}{2} h(v)} \leq \psi_{t-1,i}(G_{t-1,i}).
            \label{eq:inequality_to_show_ScInOL_2}
          \end{equation}
          When $t=\tau_i$,
          we have $v=0$ as well as $G_{t-1,i}=0$, and \eqref{eq:inequality_to_show_ScInOL_2}
          holds as its both sides are equal to $1$. For $t > \tau_i$,
          \eqref{eq:inequality_to_show_ScInOL_2} reduces to $h(v) \leq h(G_{t-1,i}/\hat{S}_{t-1,i})$,
          which holds because:
          \[
            |v| =\frac{|G_{t-1,i}|}{\sqrt{S_{t-1,i}^2 + M_{t,i}^2}}
            \leq \frac{|G_{t-1,i}|}{\sqrt{S_{t-1,i}^2 + M_{t-1,i}^2}}
            = \frac{|G_{t-1,i}|}{\hat{S}_{t-1,i}},
          \]
          and $h(x) = h(|x|)$ is monotonic in $|x|$.
         \end{proof}

        We are now ready to prove Theorem \ref{thm:alg_two}, which we restate here for
        convenience:

        \begin{theorem*}
        For any $\u \in \mathbb{R}$ the regret of $\mathrm{ScInOL}_2$ is upper-bounded by:
        \[
          \regret_T(\u)
          \leq d \epsilon +
          \sum_{i=1}^d 2|u_i| \hat{S}_{T,i} \left( \ln (3|u_i| \hat{S}^3_{T,i} \epsilon^{-1} /
          x_{\tau_i,i}^2) - 1 \right),
        \]
        where $\hat{S}_{T,i} = \sqrt{S_{T,i}^2 + M_{T,i}^2}$
         and
        $\tau_i = \min\{t \colon |x_{t,i}| \neq 0\}$.
        \end{theorem*}

        \begin{proof}
          Fixing $i \in \{1,\ldots,d\}$,
          applying Lemma \eqref{lem:ScInOL_1_progress} for $t=\tau_i,\ldots,T$, and multiplying
          over trials gives:
          \[
            \frac{\eta_{T,i}}{\eta_{\tau_i-1,i}} \geq \frac{\psi_{T,i}(G_{T,i})}
            {\psi_{\tau_i - 1,i}(G_{\tau_i-1,i})} e^{-\Delta_{T,i}},
          \]
          where we denoted $\Delta_{T,i} = \sum_{t = \tau_i}^T \delta_{t,i}$.
          From the definition of $\tau_i$, we have $\eta_{\tau_i-1,i} = \epsilon$
          and $\psi_{\tau_i-1,i} \equiv 1$. Using $\eta_{T,i} = \epsilon - \sum_{t \leq T}
          g_t x_{t,i} w_{t,i}$ we get:
          \[
            \sum_{t=1}^T g_t x_{t,i} w_{t,i}
            \leq \epsilon - \epsilon \psi_{T,i}(G_{T,i}) e^{-\Delta_{T,i}}
            \leq \epsilon - \epsilon e^{-\Delta_{T,i} + |G_{T,i}|/(2 \hat{S}_{T,i})- \frac{1}{4}
            },
          \]
          where we used \eqref{eq:h_function} to bound $h(x) \geq |x| - \frac{1}{2}$.
          By \eqref{eq:gradient_trick},
          \begin{align*}
            \regret_T(\u) &\leq \sum_{t=1}^T g_t \x_t^\top (\w_t - \u)
            =  \sum_{i=1}^d \left( \sum_{t=1}^T g_t x_{t,i}w_{t,i} +  G_{T,i} u_i \right) \\
            &\leq
            d \epsilon + \sum_{i=1}^d \left( G_{T,i} u_i -
            \epsilon e^{-\Delta_{T,i} -\frac{1}{4}} e^{|G_{T,i}|/(2 \hat{S}_{T,i})} \right)
            \\
            &\leq
            d \epsilon + \sum_{i=1}^d \sup_{x} \left\{ x u_i -
            \epsilon e^{-\Delta_{T,i} -\frac{1}{4}} e^{|x|/(2 \hat{S}_{T,i})} \right\}
            \\
            &\leq
            d \epsilon +
            \sum_{i=1}^d 2 |u_i| \hat{S}_{T,i} \left( \ln\left( 2 \epsilon^{-1}
            |u_i| \hat{S}_{T,i} e^{\frac{1}{4} + \Delta_{T,i}} \right) - 1 \right),
          \end{align*}
          where in the last inequality we used Lemma \ref{lem:conjugate_ScInOL_2} for each $i$
          with $\alpha = \epsilon e^{-\Delta_{T,i} - \frac{1}{4}}$ and $\gamma = 2 \hat{S}_{T,i}$.
          We will now show that
          \begin{equation}
            \Delta_{T,i} \leq \ln \left(\frac{\hat{S}_{T,i}^2}{x_{\tau_i,i}^2} \right),
            \label{eq:bound_on_Delta}
          \end{equation}
          which, together with $2 e^{1/4} \leq 3$ will finish the proof. To prove
          \eqref{eq:bound_on_Delta},
          we use $M_{t,i}^2 \geq x_{t,i}^2 \geq (g_t x_{t,i})^2
          = S_{t,i}^2 - S_{t-1,i}^2$ to get:
          \[
            \delta_{t,i} = \frac{(g_t x_{t,i})^2}{2(S_{t-1,i}^2 + M_{t,i}^2)}
            \leq \frac{(g_t x_{t,i})^2}{S_{t-1,i}^2 + 2 M_{t,i}^2}
            \leq \frac{(g_t x_{t,i})^2}{S_{t,i}^2 + M_{t,i}^2}
            = \frac{(M_{t,i}^2 + S_{t,i}^2) - (M_{t,i}^2 + S_{t-1,i}^2)}
            {S_{t,i}^2 + M_{t,i}^2}.
          \]
          Using $\frac{a-b}{a} \leq \ln \frac{a}{b}$ for any $a \geq b > 0$
          (which follows from the concavity of the logarithm):
          \[
            \delta_{t,i} \leq \ln \frac{M_{t,i}^2 + S_{t,i}^2}{M_{t,i}^2 + S_{t-1,i}^2}
            \leq \ln \frac{M_{t+1,i}^2 + S_{t,i}^2}{M_{t,i}^2 + S_{t-1,i}^2},
          \]
          where for $t=T$, we define $M_{T+1,i} = M_{T,i}$. Summing the above over
          trials $t = \tau_i,\ldots,T$:
          \[
            \Delta_{T,i}
            = \sum_{t=\tau_i}^T \delta_{t,i}
            \leq \ln\frac{M_{T+1,i}^2 + S_{T,i}^2}{M_{\tau_i,i}^2 + S_{\tau_i-1,i}^2}
            = \ln\frac{M_{T,i}^2 + S_{T,i}^2}{x_{\tau_i,i}^2}
            = \ln \frac{\hat{S}_{T,i}^2}{x_{\tau_i,i}^2},
          \]
          which was to be shown.
        \end{proof}

\section{Datasets}
    \label{appx:datasets}
    MNIST dataset is available at \href{http://yann.lecun.com/exdb/mnist}{Yann Lecun's page}.
    All other datasets are availableat the \href{https://archive.ics.uci.edu/ml/datasets.html}{UCI repository}.
    Scale is computed as a ratio of highest to lowest positive $L_2$ norms of features.
    \begin{table}[ht]

        \begin{tabular}{|l|l|l|l|l|}
            \hline
            \textbf{Name}&
            \textbf{features}&
            \textbf{records}&
            \textbf{classes}&
            \textbf{scale}\\
            \hline
            \hline
            Bank        &  53   &  41188    &  2    & 6.05E+05\\
            Census      & 381   &  299285   &  2    & 1.81E+06\\
            Covertype   &  54   &  581012   &  7    & 1.31E+06\\
            Madelon     & 500   &  2600     &  2    & 1.09E+00\\
            MNIST       & 728   &  70000    &  10   & 5.83E+03\\
            Shuttle     &   9   &  58000    &  7    & 7.46E+00\\
            \hline
        \end{tabular}
        \caption{Short summary of datasets}
    \end{table}

\section{Experiment: classification accuracy plots}
    \label{appx:accuracy}
    \begin{figure*}[ht!b]
            \centering
                \begin{tabular}{ccc}
                \includegraphics[width=.33\textwidth]{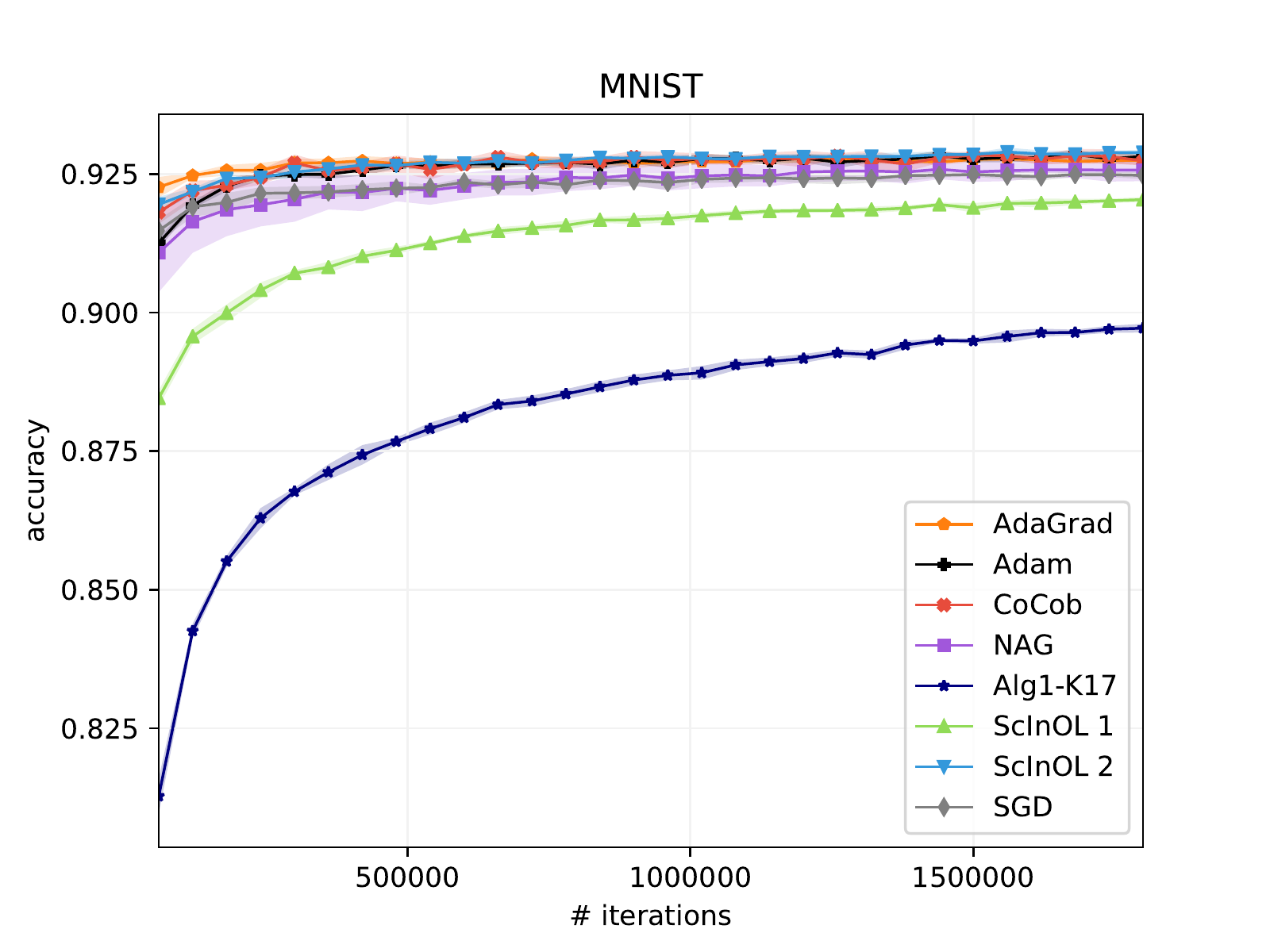} &
                \includegraphics[width=.33\textwidth]{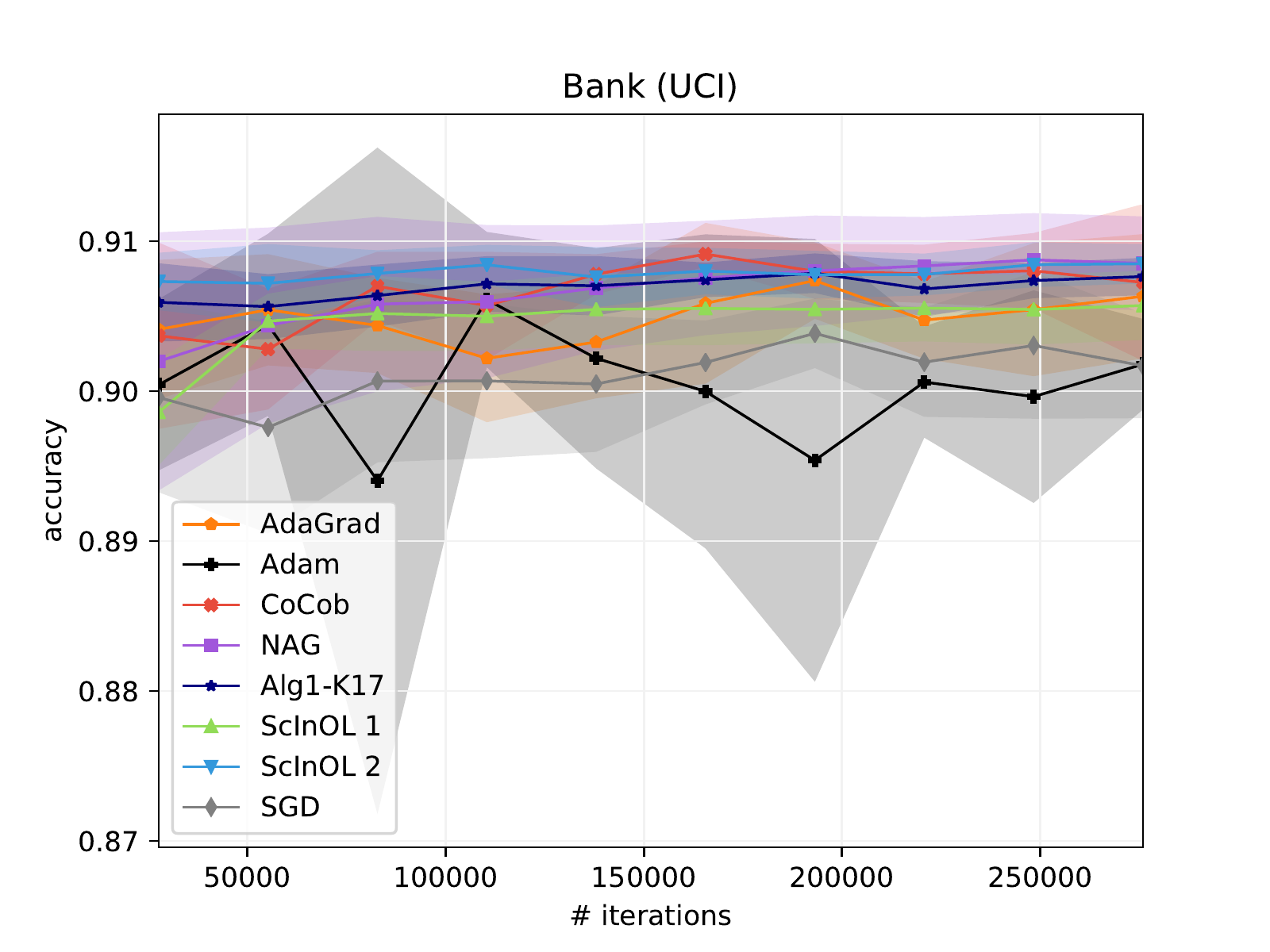} &
                \includegraphics[width=.33\textwidth]{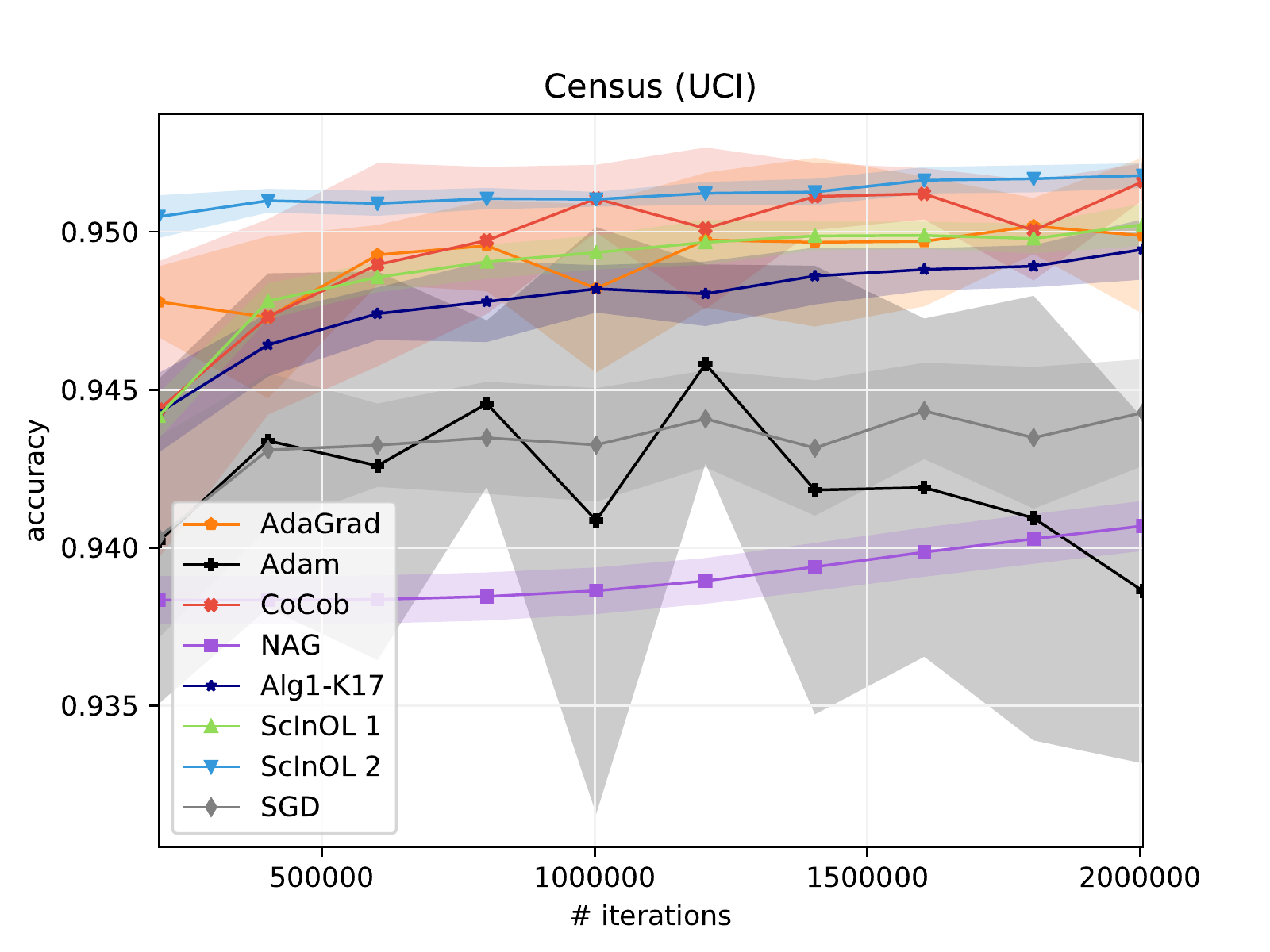}   \\
                \includegraphics[width=.33\textwidth]{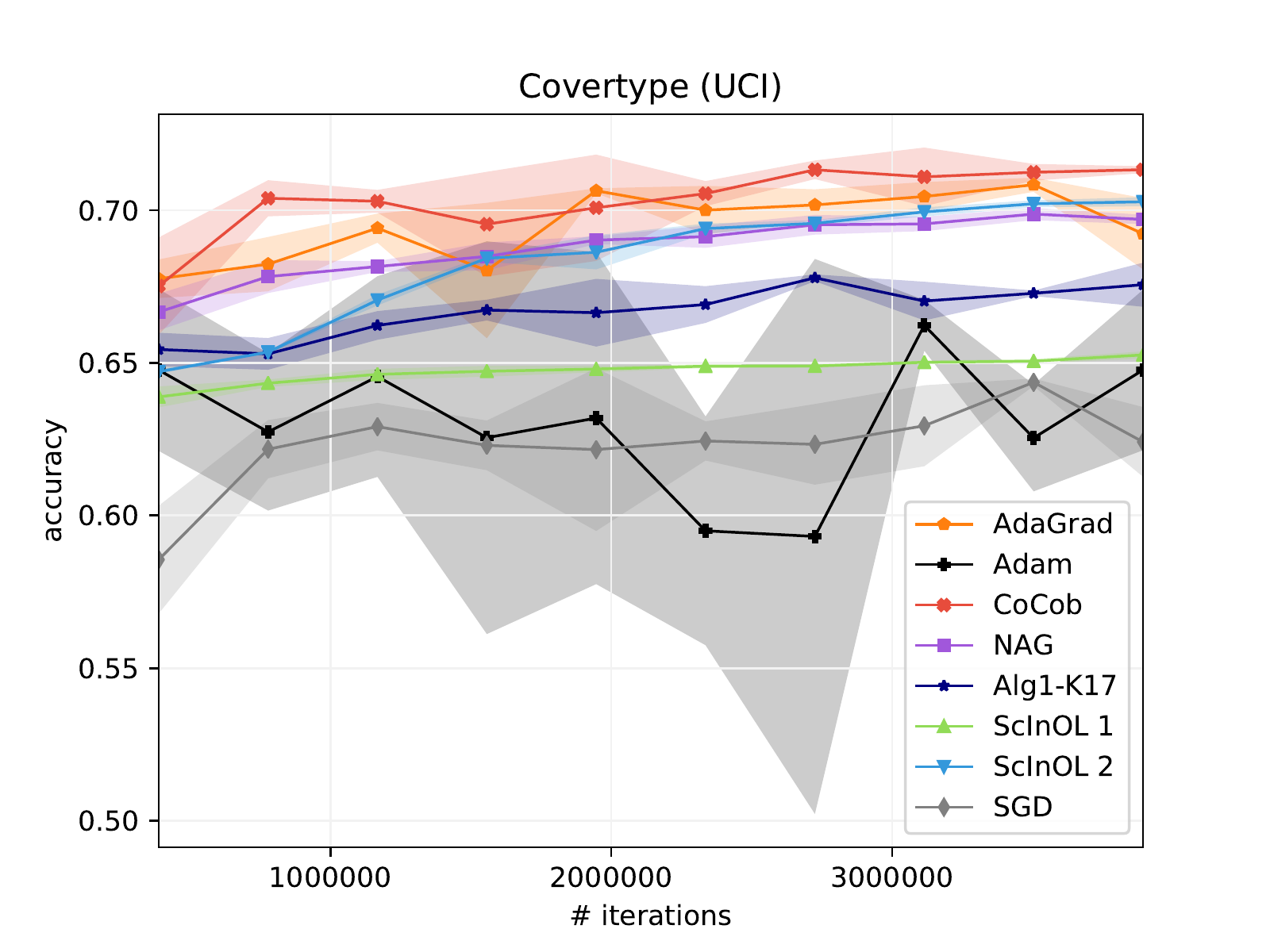} &
                \includegraphics[width=.33\textwidth]{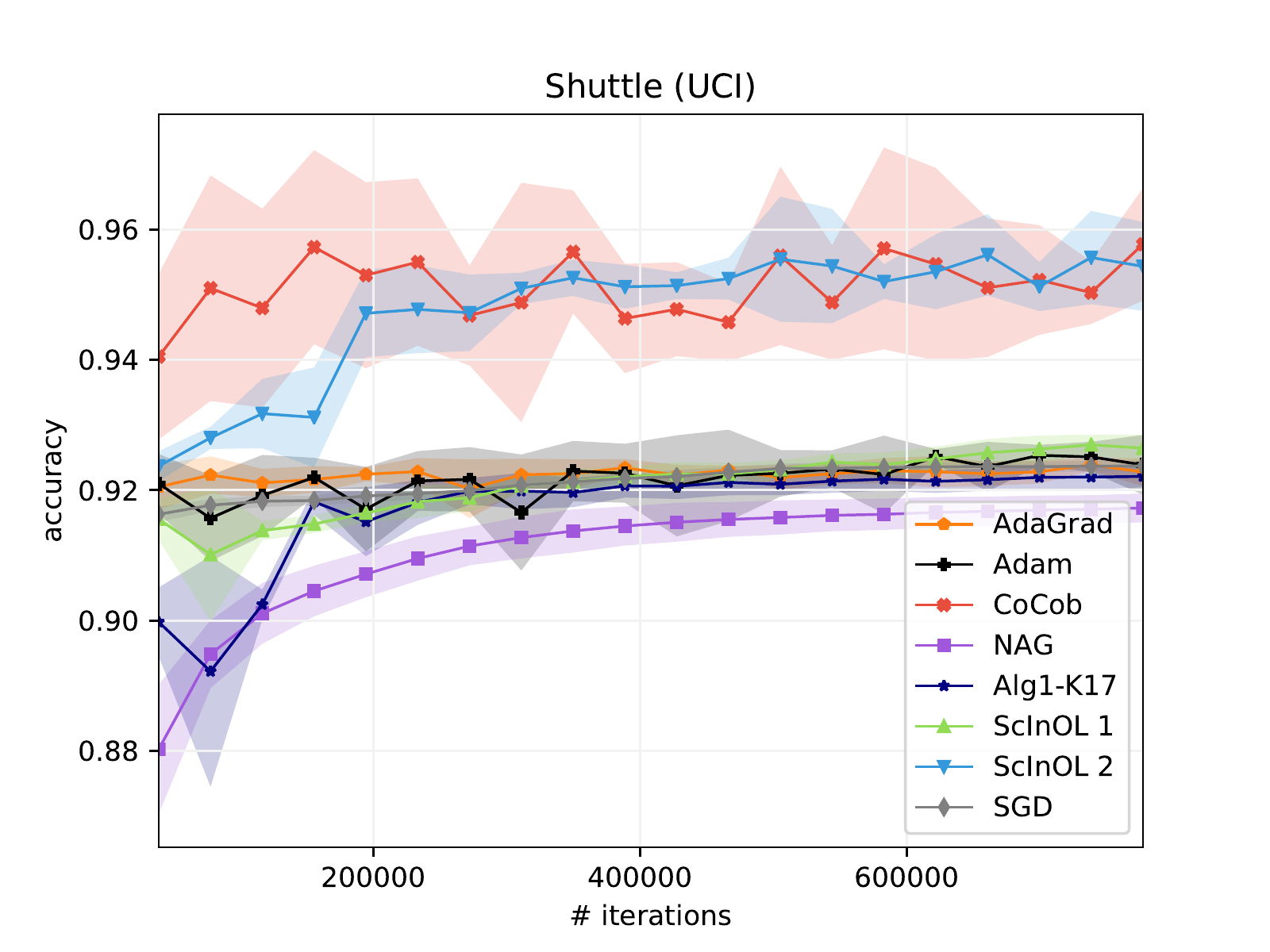} &
                \includegraphics[width=.33\textwidth]{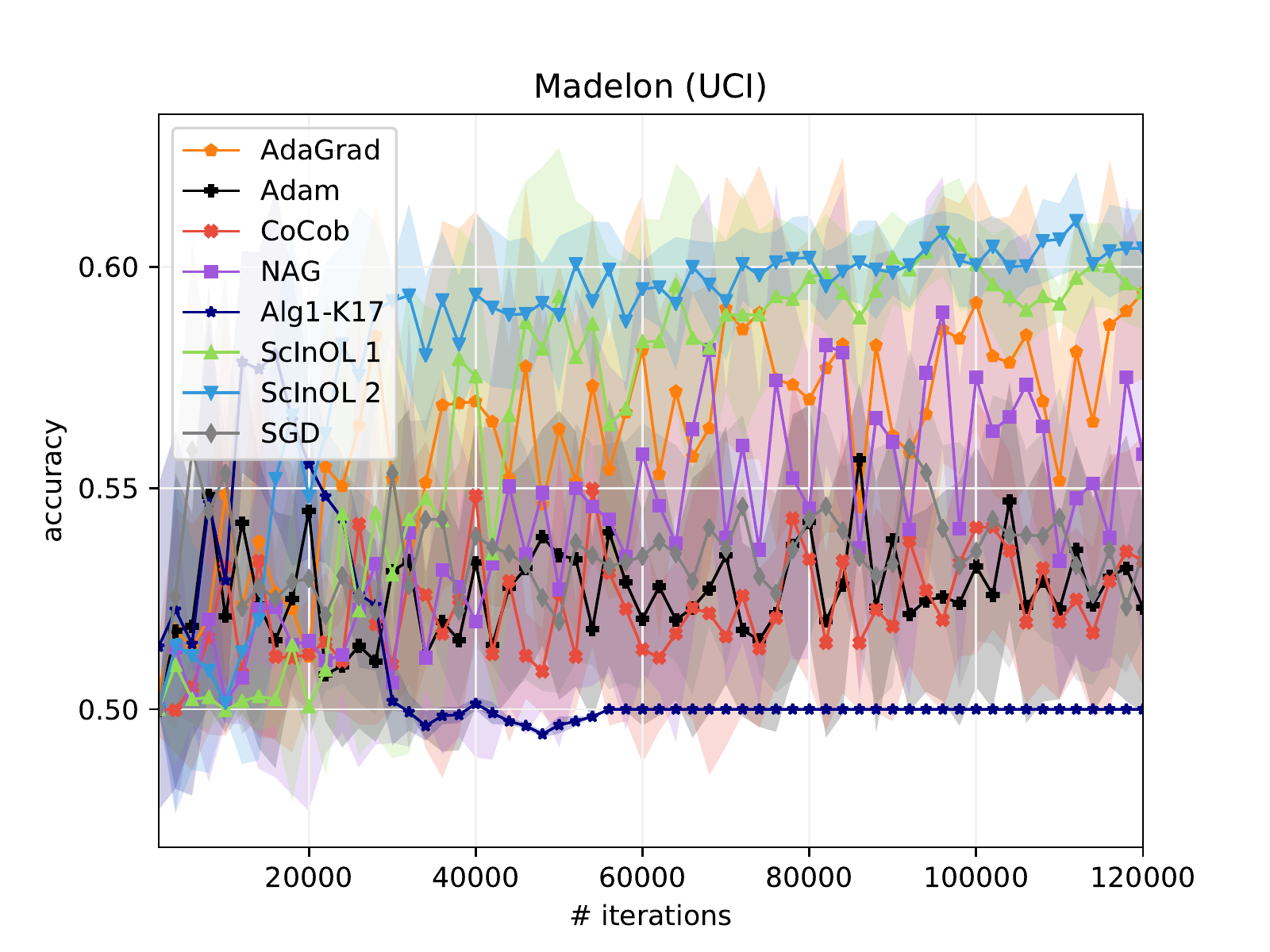}   \\
                \end{tabular}
                \caption{Accuracy results for linear classification experiments.}
            \end{figure*}

\section{Multivariate predictions}
\label{appendix:multivariate_predictions}

For simplicity, in the paper we focus on loss functions defined for real-valued predictions
$\hy \in \mathbb{R}$. Sometimes, however, it is natural to consider a setup of multivariate predictions
$\hby \in \mathbb{R}^K$. For instance, the multinomial logistic loss (cross-entropy loss)
is defined for $y \in \{1,\ldots,K\}$ as:
\[
\ell(y, \hby) = - \sum_{k=1}^K \boldsymbol{1}[y=k] \ln \sigma_k(\hby)
= - \hy_{y} + \ln \left( \sum_{k=1}^K e^{\hy_k} \right),
\]
where $\sigma_k(\hby) = \frac{e^{\hy_k}}{\sum_{j=1}^K e^{\hy_j}}$ is the soft-max transform.

We assume the multivariate losses $\ell_t(\hby) = \ell(y_t, \hby)$ are
convex and $L$-Lipschitz in the sense that the max-norm of 
subgradient $\nabla \ell_t(\hby)$ for any $\hby$ is bounded, 
$\|\nabla \ell_t(\hby)\|_{\infty} \leq L$ (which is satisfied with $L=1$ by the multinomial
logistic loss).
We consider the class of comparators which are parameterized
by $\U \in \mathbb{R}^{d \times K}$, a $d \times K$ 
parameter matrix, and the
regret of the algorithms against $\U$ for a sequence of data $\{(\x_t,y_t)\}_{t=1}^T$
is defined as:
\[
\regret_T(\U) = \sum_{t=1}^T \ell_t(\hby_t) - \sum_{t=1}^T 
\ell_t(\U^\top \x_t).
\]
Consider an algorithm which at trial $t$ predicts with a weight matrix 
$\W_t \in \mathbb{R}^{d \times K}$, $\hby_t = \W_t^\top \x_t$.
Using the convexity of the loss, for any $\hby, \hby'$ and any $t$ we have
$\ell_t (\hby') \geq \ell_t(\hby) + \nabla \ell_t(\hby)^\top (\hby' - \hby)$.
Denoting $\nabla \ell_t(\hby_t)$ by 
$\boldsymbol{g}_t = (g_{t,1},\ldots,g_{t,K})$ with $g_{t,k} \in [-L,L]$
for all $k=1,\ldots,K$,
and using the bound above with $\hby = \hby_t = \W_t^\top \x_t$
and $\hby' = \U^\top \x_t$ we have:
\[
\regret_T(\U) = \sum_{i=1}^d \sum_{k=1}^K 
\left( \sum_{t=1}^T g_{t,k} x_{t,i} (W_{t;i,k} - U_{i,k}) \right).
\]
The regret decouples into a sum over individual coordinates and dimensions of the prediction
vector, and the extension of our algorithms is now straightforward
(see Algorithm \eqref{alg:one_multi} and \eqref{alg:two_multi} below). Also, the analysis
can be carried out in full analogy to the univariate loss case resulting in the following
bounds (for $L=1$):
\begin{theorem}
For any $\U \in \mathbb{R}^{d \times K}$ 
the regret of $\mathrm{ScInOL}_1$ is upper-bounded by:
\[
  \regret_T(\U)
   \leq \sum_{i=1}^d \sum_{k=1}^K \left(2|U_{i,k}| \hat{S}_{T;i,k} \ln (1 + 2|U_{i,k}| 
   \hat{S}_{T;i,k} \epsilon^{-1} T) + \epsilon(1 + \ln T )\right) 
\]
where $\hat{S}_{T;i,k} = \sqrt{S_{T;i,k}^2 + M_{T;i}^2}$.
\end{theorem}
\begin{theorem}
For any $\U \in \mathbb{R}^{d \times K}$ 
the regret of $\mathrm{ScInOL}_2$ is upper-bounded by:
\[
  \regret_T(\U)
  \leq d K \epsilon +
  \sum_{i=1}^d \sum_{k=1}^K 2|U_{i,k}| \hat{S}_{T;i,k} \left( \ln (3|U_{i,k}| 
  \hat{S}^3_{T;i,k} \epsilon^{-1} /
  x_{\tau_i,i}^2) - 1 \right),
\]
    where $\hat{S}_{T;i,k} = \sqrt{S_{T;i,k}^2 + M_{T;i}^2}$
 and
$\tau_i = \min\{t \colon |x_{t,i}| \neq 0\}$.
\end{theorem}

\begin{algorithm2e}%
\DontPrintSemicolon
\SetAlgoNoEnd
\SetKwInOut{Initialization}{Initialization}
\Initialization{$S^2_{0;i,k}, G_{0;i,k}, M_{0;i} \leftarrow 0, \beta_{0;i,k} 
\leftarrow \epsilon \; (i=1,\ldots,d; \; k=1,\ldots,K)$}
 \For{$t=1,\ldots,T$}{
   Receive $\x_t \in \mathbb{R}^d$\;
   \For{$i=1,\ldots,d$}{
     $M_{t;i} \leftarrow \max \{M_{t-1;i}, |x_{t,i}|\}$\;
     \For{$k=1,\ldots,K$}{
     $\beta_{t;i,k} \leftarrow \min\{\beta_{t-1;i,k}, \epsilon (S^2_{t-1;i,k} + M_{t;i}^2)/(x^2_{t,i} t)\}$ \;
     $W_{t;i,k} = \frac{\beta_{t;i,k} \mathrm{sgn}(\theta_{t;i,k})}{2\sqrt{S_{t-1;i,k}^2 + M_{t;i}^2}}
        \Big(e^{|\theta_{t;i,k}|/2} - 1 \Big)$, \qquad where
   $\theta_{t;i,k} = \frac{G_{t-1;i,k}}{\sqrt{S^2_{t-1;i,k} + M^2_{t;i}}}$}
   }
   Predict with $\hby_t = \W_t^\top \x_t$, receive loss $\ell_t(\hby_t)$ and compute 
   $\boldsymbol{g}_t = \nabla_{\hby_t} \ell_t(\hby_t)$\;
   \For{$i=1,\ldots,d$}{
      \For{$k=1,\ldots,K$} {
     $G_{t;i,k} \leftarrow G_{t-1;i,k} - g_{t,k} x_{t,i}$ \;
     $S^2_{t;i,k} \leftarrow S^2_{t-1;i,k} + (g_{t,k} x_{t,i})^2$\;
   }}
 }
\caption{$\text{ScInOL}_1 (\epsilon)$ for multivariate losses}%
\label{alg:one_multi}
\end{algorithm2e}%

\begin{algorithm2e}%
\DontPrintSemicolon
\SetAlgoNoEnd
\SetKwInOut{Initialization}{Initialization}
\Initialization{$S^2_{0;i,k}, G_{0;i,k}, M_{0;i} \leftarrow 0, \eta_{0;i,k} 
\leftarrow \epsilon \; (i=1,\ldots,d; \; k=1,\ldots,K)$}
 \For{$t=1,\ldots,T$}{
   Receive $\x_t \in \mathbb{R}^d$\;
   \For{$i=1,\ldots,d$}{
     $M_{t;i} \leftarrow \max \{M_{t-1;i}, |x_{t,i}|\}$\;
     \For{$k=1,\ldots,K$}{
 $W_{t;i,k} = \frac{\mathrm{sgn}(\theta_{t;i,k}) \min\{|\theta_{t;i,k}|,1\}}{2 \sqrt{S_{t-1;i,k}^2 + M_{t;i}^2}} \eta_{t-1;i,k}$,
     \qquad where $\theta_{t;i,k} = \frac{G_{t-1;i,k}}{\sqrt{S^2_{t-1;i,k} + M^2_{t;i}}}$}
     }
   Predict with $\hby_t = \W_t^\top \x_t$, receive loss $\ell_t(\hby_t)$ and compute 
   $\boldsymbol{g}_t = \nabla_{\hby_t} \ell_t(\hby_t)$\;
   \For{$i=1,\ldots,d$}{
     \For{$k=1,\ldots,K$} {
     $G_{t;i,k} \leftarrow G_{t-1;i,k} - g_{t,k} x_{t,i}$ \;
     $S^2_{t;i,k} \leftarrow S^2_{t-1;i,k} + (g_{t,k} x_{t,i})^2$\;
     $\eta_{t;i,k} \leftarrow \eta_{t-1;i,k} - g_{t,k} x_{t,i} w_{t,i,k}$
   }}
 }
\caption{$\text{ScInOL}_2 (\epsilon)$ for multivariate losses}%
\label{alg:two_multi}
\end{algorithm2e}%

\end{document}